\documentclass{article}
\usepackage{ijcai22}
\usepackage{times}
\usepackage{soul}
\usepackage{url}
\usepackage[hidelinks]{hyperref}
\usepackage[utf8]{inputenc}
\usepackage[small]{caption}
\usepackage{graphicx}
\usepackage{amsmath}
\usepackage{amsthm}
\usepackage{booktabs}
\usepackage{algorithm}
\usepackage{algorithmic}
\usepackage{tasks}
\settasks{
  style=itemize,
  after-item-skip=0pt
}

\usepackage{newfloat}
\usepackage{listings}

\usepackage{amsmath,amscd,amsbsy,amssymb,latexsym,url,bm,amsthm}
\allowdisplaybreaks
\usepackage{cleveref}
\usepackage{verbatim}
\usepackage{color}
\usepackage{dsfont}
\usepackage{caption}
\usepackage{subcaption}

\newtheorem{theorem}{Theorem}

\newtheorem{lemma}{Lemma}

\newtheorem{definition}{Definition}

\newcommand{\cA}{\mathcal{A}}
\newcommand{\cB}{\mathcal{B}}

\newcommand{\cE}{\mathcal{E}}
\newcommand{\cF}{\mathcal{F}}
\newcommand{\cH}{\mathcal{H}}

\newcommand{\cN}{\mathcal{N}}

\newcommand{\cK}{\mathcal{K}}
\newcommand{\Beta}{\mathrm{Beta}}

\newcommand{\abs}[1]{\left| #1 \right|}
\newcommand{\bOne}[1]{\mathds{1} \! \left\{#1\right\}}
\newcommand{\bracket}[1]{\left(#1\right)}

\newcommand{\set}[1]{\left\{ #1 \right\}}
\newcommand{\EE}[1]{\mathbb{E} \left[#1\right]}
\newcommand{\PP}[1]{\mathbb{P} \left(#1\right)}

\newif\ifsup\suptrue
\usepackage{tikz}
\usetikzlibrary{shapes,shapes.misc,positioning}
\usetikzlibrary{patterns,arrows,decorations.pathreplacing}

\DeclareMathOperator*{\argmax}{argmax}

\title{Thompson Sampling for Bandit Learning in Matching Markets}

\author{
Fang Kong$^1$
\and
Junming Yin$^2$\And
Shuai Li$^{1}$\thanks{Corresponding author. The full version is available at \url{http://arxiv.org/abs/2204.12048}.
}
\affiliations
$^1$John Hopcroft Center for Computer Science, Shanghai Jiao Tong University\\
$^2$Tepper School of Business, Carnegie Mellon University
\emails
\{fangkong, shuaili8\}@sjtu.edu.cn,
junmingy@cmu.edu
}
\begin{document}

\maketitle
%  A line of recent works has focused on the problem setting where the preferences of one-side market participants are unknown a priori and need to be learned from iterative interaction with the other side of participants. 
%   Thompson sampling (TS) is another popular approach that has attracted a lot of attention because of its easier implementation and better empirical performances. 
% [In many problems, even when UCB and ETC-type algorithms have already been analyzed, people are still trying to study TS for its benefits.]
% However, the convergence analysis of TS is much more challenging and
% remains open in many cases. 
% In this paper, we provide the first analysis of regret bound for TS in the new setting of [iterative matching markets]. 
\begin{abstract}
The problem of two-sided matching markets has a wide range of real-world applications and has been extensively studied in the literature. A line of recent works have focused on the problem setting where the preferences of one-side market participants are unknown \emph{a priori} and are learned by iteratively interacting with the other side of participants. All these works are based on explore-then-commit (ETC) and upper confidence bound (UCB) algorithms, two common strategies in multi-armed bandits (MAB). Thompson sampling (TS) is another popular approach, which attracts lots of attention due to its easier implementation and better empirical performances. In many problems, even when UCB and ETC-type algorithms have already been analyzed, researchers are still trying to study TS for its benefits. However, the convergence analysis of TS is much more challenging and remains open in many problem settings. In this paper, we provide the first regret analysis for TS in the new setting of iterative matching markets. Extensive experiments demonstrate the practical advantages of the TS-type algorithm over the ETC and UCB-type baselines. 

% \fang{All these works are based on explore-then-commit (ETC) and upper confidence bound (UCB) algorithms, which are common in multi-armed bandits (MAB). 
% Thompson sampling (TS) is another popular approach. It attracts lots of attention due to easier implementation and better empirical performances.}
% The setting of matching markets brings new challenges for learners, so it is an open question whether TS can work in this setting.
% In this paper, we investigate this research question and propose the first TS-type algorithm for this setting with a rigorous theoretical guarantee. 

\end{abstract}

%!TEX root =  main.tex

\section{Introduction}
\label{sec:intro}
The model of matching markets has been studied for several decades \cite  {gale1962college,roth1997turnaround,haeringer2011decentralized}. It has a wide range of applications including labor employment \cite{roth1984evolution}, house allocation \cite{abdulkadirouglu1999house}, and college admission \cite{epple2006admission,fu2014equilibrium}. 
Typically there are two sides of players, such as employers and workers in the labor market, and each player of one side has a preference ranking over players on the other side.  
Stability is a key concept in matching markets, which ensures market participants have no incentive to abandon the current partner \cite{gale1962college,roth1992two}. 
% kelso1982job,
% Preserving the stability maintains the robustness of the market and 
% % promotes the long-term sustainability, which also 
% eliminates the justified envy to guarantee fairness \cite{roth2002economist,abdulkadirouglu2003school}.  
Many researchers study how to find a stable matching in the markets \cite{gale1962college,roth1992two}. 
However, most of them assume the full preferences of players are known \emph{a priori}, which is not realistic in many real-world applications.  
% the market participants are usually uncertain about their preferences in real applications. 
For example, the demand-side players in online matching markets (such as employers in UpWork or Mechanical Turk) are likely to be uncertain about the qualities of supply-side players (such as workers).
% The employers usually have a stream of similar tasks to be delegated
Many such online platforms usually have matching processes happen iteratively, which allows the participating players to adaptively learn their unknown preferences \cite{liu2020competing,liu2020bandit}.
% \shuai{ }

% Fortunately, the online platform makes it possible for players to iteratively match with others and the unknown preferences are learned during these processes 

% they can iteratively match with each other and learn these unknown preferences. 
% These players can iteratively match with others and their unknown preferences can thus be learned during these processes. 
% \fang{
% The players can adaptively learn their unknown preferences in the iterative matching processes . }
% The iterative matching processes let agents to learn their unknown preferences adaptively. 
% These online platforms give agents opportunities to iteratively match with supply-side agents and the unknown preferences can thus be learned during this process.  
% Some recent works study this direction \cite{liu2020competing,liu2020bandit,sankararaman2021dominate,basu21beyond}.

% But these unknown preferences can be learned from observations during iterative interactions with other market participants. 
% For example, employers usually do not know the exact skill level of a worker before the worker finishes the task, 

Multi-armed bandits (MAB) is a common approach to modeling this type of learning process \cite{auer2002finite,lattimore2020bandit}. 
The most basic framework considers a single player and $K$ arms, where the player does not have prior knowledge over arms and will learn it through iteratively collected rewards. The objective of the player is to maximize the cumulative reward over a specified horizon, or equivalently minimize the cumulative regret, which is defined as the difference between the cumulative reward of the optimal arm and that of the chosen arms. 
To achieve this goal, the player needs to make a trade-off between exploration and exploitation: the former tries arms that have not been observed enough times to get potential higher rewards, and the latter focuses on arms with the highest observed rewards so far to maintain high profits. 
% How to balance the exploration-exploitation trade-off is the core of the player's strategy. 

There are many types of strategies to balance such trade-offs with theoretical guarantees, among which explore-then-commit (ETC), upper confidence bound (UCB), and Thompson sampling (TS) are widely adopted in the literature \cite{lattimore2020bandit}. 
The TS algorithm was introduced in the 1930s \cite{thompson1933likelihood} but has not been theoretically proven until the 2010s \cite{kaufmann2012thompson,agrawal2012analysis,agrawal2013further}. 
Algorithms of this type form a competitive family in the bandit area because of many advantages such as easier implementations and better practical performances. Though ETC and UCB-type algorithms have solved many bandit problems, there are still accompanying works trying to analyze TS \cite{cheung2019thompson,kong2021hardness}.

% However, the analysis of TS is more challenging and remains open in many settings \cite{wang2018thompson,kong2021hardness}.

% In many bandit problems, even though ETC and UCB-type algorithms have been already proved, analyzing a TS-type one is still worthy 
% Benefiting from its merit of easier implementation and better practical performances, the TS-type algorithm has recently attracted a lot of attention even though some ETC and UCB-type algorithms may have been already proposed to solve the same problem  .

% \fang{}
\citeauthor{das2005two} \shortcite{das2005two} first introduce the bandit learning problem in two-sided matching markets and show the empirical performances of algorithms via simulation tests. 
\citeauthor{liu2020competing} \shortcite{liu2020competing} later study a refinement of the problem and give the first algorithm with theoretical guarantees. They propose both ETC and UCB-type algorithms with upper bounds on the \textit{stable regret}, which is defined as the difference between the cumulative reward of a stable matching\footnote{Different stable matchings give different kinds of stable regrets.} and the cumulative reward collected. 
Both algorithms adopt a central platform to collect players' preferences and assign allocations to participants. 
Since such platforms do not always exist in real-world applications, the following works extend to the decentralized setting \cite{liu2020bandit,sankararaman2021dominate,basu21beyond}. All these works still focus on ETC and UCB-type algorithms and do not consider TS. 

In a standard TS analysis, one needs to bound the inaccurate estimations for the optimal arms due to the randomness caused by the posterior samples, which is easily controlled because eventually the optimal arms can be observed enough times. 
However, this property may not hold anymore in matching markets since the observations are not simply decided by the number of selections and can be blocked if the selected arm rejects this player. 
Such a difficulty does not exist in UCB-type algorithms since it mainly needs to bound inaccuracies for sub-optimal arms, and ETC-type algorithms can control the blockings easily in the design.

In this work, we present the first TS-type algorithm for decentralized matching markets. 
For the stable regret analysis, though the observations might be blocked, we prove that its success probabilities could be bounded by separately analyzing the influence of other players in a fine division of events with a different number of selections on the optimal stable arms.
The result guarantees the regret of TS with an order of $O(\log^2 T/\Delta^2)$, where $T$ is the horizon and $\Delta$ is the minimum preference gap. 
A series of experiments are conducted to show that the practical performances of the TS-type algorithm are better than other baselines.

\paragraph{Related Work}
\label{sec:related}

The MAB problem has been studied for many decades, which captures the learning process of a single player in an unknown environment.
% \cite{lattimore2020bandit}. 
The study on stochastic MAB is an important line of works where the reward of arms are drawn from fixed distributions. $\epsilon$-greedy, ETC, UCB and TS are all classical algorithms for this setting and are widely followed in the literature \cite{lattimore2020bandit}. 
% The objective of algorithms is to minimize the expected cumulative regret, defined as the difference between the cumulative reward of the optimal arm and the cumulative reward received by the player. 

Motivated by real applications including cognitive radio where multiple players compete and cooperate in the unknown environment, the problem extends to multi-player MAB. 
These works \cite{liu2010distributed,rosenski2016multi,bistritz2018distributed} mainly consider the case where a player receives no reward if it conflicts with others by selecting the same arm and measure algorithms performances by the collective cumulative regret of all players.

% Some works consider scenarios where all players share the same preferences over arms and the player would receive no feedback if it collides with others by choosing the same arm \cite{liu2010distributed,rosenski2016multi}.
% % \cite{liu2010distributed,anandkumar2011distributed,rosenski2016multi}. 
% Later, the setting is generalized by considering different preferences of players \cite{nayyar2016regret,bistritz2018distributed}. 
% % \cite{kalathil2014decentralized,nayyar2016regret,bistritz2018distributed}. 
% In these problems, the performances of algorithms are measured by the collective cumulative regret of all players. 
% % , defined as the cumulative difference of the collected reward of all players from the collective reward of the optimal assignment. 

The model of combinatorial dueling bandits \cite{chen2020combinatorial} can also be regarded as a multi-player problem, where the preferences of players are defined on pairs of arms. 
% to represent the preference probability of one over the other. 
% \citeauthor{chen2020combinatorial} \shortcite{chen2020combinatorial} 
They study a pure exploration problem for this model with the objective to find the best matching after a specified period.
% In each round, the algorithm selects a duel of two arms for each player and observes which arm wins. 
% to learn the preference probabilities. 

The two-sided matching market problem is different from previous multi-player MAB by considering arms' preferences \cite{gale1962college,roth1997turnaround}. In this problem, not only do players have arbitrary preferences over arms, arms also have arbitrary preferences over players. When multiple players collide at the same arm, the player preferred most by this arm will receive the corresponding reward while others receive no feedback.

% \fang{}

The bandit learning problem in two-sided matching markets was first introduced by \citeauthor{das2005two} \shortcite{das2005two}. They assume the preferences of both sides are unknown and study the empirical performances of $\epsilon$-greedy in a special market where all participants on one side share the same preferences. 
% They assume the preferences of participants in both sides need to be learned and show the empirical performances of $\epsilon$-greedy in the market with global preferences, where all participants on one side share the same preferences. 
Later, \citeauthor{liu2020competing} \shortcite{liu2020competing} study a variant of the problem by considering one-side unknown preferences. 
% two-sided matching market in the bandit setting and aim to minimize the cumulative stable regret of each player.
They propose both ETC and UCB-type algorithms and show the convergence analysis of the stable regret. 
Both algorithms adopt Gale-Shapley (GS) \cite{gale1962college} to assign allocations for players. 
% With the help of the Gale-Shapley (GS) algorithm \cite{gale1962college}, a classical algorithm to find a stable matching in the markets,  \citeauthor{liu2020competing} \shortcite{liu2020competing} study a non-collision case where GS collect players' estimated preferences and assigns allocations to players. 
% Both ETC and UCB-type algorithms are proposed to solve this problem \cite{liu2020competing}. 
However, in real applications, players usually prefer to independently make decisions. 
% The problem then extends to the decentralized matching markets. 
\citeauthor{liu2020competing} \shortcite{liu2020competing} then propose Decentralized ETC in the decentralized setting, which lets players to explore arms for a fixed number of rounds. However, the exploration budget $H$ needs to depend on the preference gap, which is usually not known beforehand. 
 \citeauthor{sankararaman2021dominate} \shortcite{sankararaman2021dominate} and \citeauthor{basu21beyond} \shortcite{basu21beyond} successively study the decentralized setting to remove this assumption but under special assumptions to guarantee a unique stable matching.
For general markets, \citeauthor{basu21beyond} \shortcite{basu21beyond} propose the phasedETC algorithm and \citeauthor{liu2020bandit} \shortcite{liu2020bandit} propose a UCB-type algorithm to avoid conflicts among players and minimize the stable regret. 
\citeauthor{dai2020learning} \shortcite{dai2020learning} also study to learn players' unknown preferences in a decentralized matching market. However, their objective is to estimate the unknown preferences using a statistical model but involves no cumulative regret.
To the best of our knowledge, we are the first to study the TS-type algorithm for two-sided matching markets. 

%!TEX root =  main.tex

% \fang{to be done:
% 1. 1-subgaussian results
% 2. novelty discussion
% 4.experiments
% }

\section{Setting}
\label{sec:setting}

There are $N$ players and $K$ arms in the market. The player set is denoted by $\cN = \set{p_1,\ldots,p_N}$ and the arm set is denoted by $\cK = \set{a_1,\ldots,a_K}$. To ensure each player can be matched with an arm, we assume $N \le K$ as \cite{liu2020competing,liu2020bandit,basu21beyond,sankararaman2021dominate}. 

For each player $p_i$, its preference for arm $a_j$ is quantified by an unknown value $\mu_{i,j} \in [0,1]$. For two different arms $a_j$ and $a_{j'}$, $\mu_{i,j}>\mu_{i,j'}$ implies that player $p_i$ \textit{truly} prefers arm $a_j$ to $a_{j'}$. 
Similarly, each arm $a_j$ has preferences over players.  
Let $\pi_{j,i}$ to represent the preference value of arm $a_j$ for player $p_i$. For two different players $p_i$ and $p_{i'}$, $\pi_{j,i} > \pi_{j,i'}$ implies that arm $a_j$ prefers player $p_{i}$ to $p_{i'}$.
The ranking for each arm $a_j$'s preferences $(\pi_{j,i})_{i \in [N]}$ is assumed to be known since it usually can be estimated by some known utilities such as the payments given by employers in the labor market.

At each round $t=1,2,\ldots$, each player $p_i \in \cN$ attempts to pull an arm $A_{i}(t) \in \cK$. Let $A(t)=(A_i(t))_{i \in [N]}$. 
When multiple players attempt to pull the same arm, there will be a conflict and only the player preferred most by this arm is accepted. 
If a player $p_i$ wins the conflict, it will receive a random reward $X_{i,A_{i}(t)}(t) \in [0,1]$ with expectation $\mu_{i,A_{i}(t)}$. Other players $p_{i'}$ who fail the conflict are unmatched in this round and receive $X_{i',A_{i'}(t)}(t)=0$.
Following the observations in the scenario of the labor market and college admission where the arm side (e.g., workers or colleges) usually announces the successfully matched players (e.g., employers or students), we assume the successfully matched player for each arm is public at the end of the round as in previous work \cite{liu2020bandit}. 

To measure the status of the market, stable matching \cite{gale1962college} is introduced to define an equilibria. A stable matching is a one-to-one mapping from players to arms without a pair of player and arm such that they both prefer being matched with each other over the current partner.  
With the true preference rankings of players and arms, there may be more than one stable matching. Denote $m_i$ as player $p_i$'s least favorite arm among matched arms from all stable matchings. 
% that can be matched with $p_i$ in a stable matching.  
The objective is to find a stable matching between players and arms and minimize the cumulative stable regret for each player $p_i$ \cite{liu2020competing,liu2020bandit,basu21beyond,sankararaman2021dominate}, which is defined as 
\begin{align}
	R_i(T) = T\mu_{i,m_i} -  \EE{ \sum_{t=1}^T  X_{i,A_{i}(t)}(t) }    \,,
\end{align}
where the expectation is taken over the randomness in the reward payoffs and the algorithm.

\paragraph{Motivating example} 
Common real-world applications of above setting include online labor market Upwork and online crowd-sourcing platform Amazon Mechanical Turk as shown in \citeauthor{liu2020competing} \shortcite{liu2020competing}. 
The employers in these platforms have numerous similar tasks to be delegated, and the workers can only work on one task at a time. 
The reward of the employer when being matched with a worker corresponds to how well the task was done. Employers usually have unknown preferences over workers since they do not know their real skill levels, while workers’ preferences over employers can be certain based on the payment or their familiarity with tasks. Thus the employers can be modeled as players and workers can be modeled as arms in our setting. The employers would provide repeated offers to match with workers through these numerous tasks and learn their uncertain preferences during this process. When faced with multiple offers at a time, the worker would choose the most preferred one.

%!TEX root =  main.tex

\section{Algorithm}
\label{sec:alg}

In this section, we introduce our TS-type algorithm, referred to as conflict-avoiding TS (CA-TS, Algorithm \ref{alg:CA-TS}), for two-sided matching markets. 
Some design ideas are motivated by \citeauthor{liu2020bandit} \shortcite{liu2020bandit} and the main difference is that CA-TS uses parameters sampled from posterior distributions as estimations for preferences to select arms.
However, this raises new analysis difficulties since CA-TS additionally requires accurate estimations for `optimal arms' to ensure a stable matching, which will be made clearer later in Section \ref{sec:analysis}.

\begin{algorithm}[thb!]
    \caption{Conflict-Avoiding TS (CA-TS)}\label{alg:CA-TS}
    \begin{algorithmic}[1]
    \STATE Input: Player set $\cN$, arm set $\cK$, parameter $\lambda \in (0,1)$ \label{alg:cats:input}; 
    \STATE Initialize: $\forall i \in[N], j \in [K], a_{i,j}=b_{i,j}=1$ \label{alg:cats:initial}
    \FOR{$t=1,2,\cdots$}
        \FOR{$p_i \in \cN$}
        \STATE $\forall a_j$, sample $\theta_{i,j}(t)\sim\Beta(a_{i,j},b_{i,j})$ \label{alg:cats:sample}
            \STATE Independently draw $D_{i}(t)\sim\mathrm{Bernoulli}(\lambda)$ \label{alg:cats:draw}
            \IF{$D_{i}(t)==0$} \label{alg:cats:select:start}
                \STATE Construct plausible set \label{alg:cats:plausible}\\
                $$ S_{i}(t) := \set{j: \pi_{j,i}\ge \pi_{j,i'} \text{ where } \bar{A}_{i'}(t-1)=j  }$$
                \STATE Pull $A_{i}(t) \in \argmax_{j \in S_{i}(t)}\theta_{i,j}(t)$ \label{alg:cats:select:end}
            \ELSE  
                \STATE Pull $A_{i}(t) = A_{i}(t-1)$ \label{alg:cats:delay}
            \ENDIF
            \IF{$p_i$ wins conflict} \label{alg:cats:update:start}
                \STATE $\bar{A}_{i}(t) = A_{i}(t)$ \label{alg:cats:mark}
%                 \STATE $ Y_i(t)=\left\{
% \begin{aligned}
% 1,&  ~\text{with probability } X_{i,A_{i}(t)}(t)\\
% 0,&  ~\text{with probability } 1-X_{i,A_{i}(t)}(t)
% \end{aligned}
% \right.
% $
                \STATE $Y_i(t) \sim \mathrm{Bernoulli}\bracket{X_{i,A_{i}(t)}(t)}$
                % With probability $X_{i,A_{i}(t)}(t),\ Y_{i}(t) = 1$;\\ With probability $1-X_{i,A_{i}(t)}(t),\ Y_{i}(t) = 0$
                \STATE Update $a_{i,A_{i}(t)} = a_{i,A_{i}(t)} +Y_{i}(t)$\\$b_{i,A_{i}(t)} = b_{i,A_{i}(t)}+(1-Y_{i}(t))$
            \ELSE
                \STATE $\bar{A}_{i}(t) = -1$
            \ENDIF \label{alg:cats:update:end}
        \ENDFOR
    \ENDFOR
    \end{algorithmic}
\end{algorithm}

The CA-TS algorithm takes the player set $\cN$ and arm set $\cK$ as input (line \ref{alg:cats:input}). 
For each player $p_i$ and arm $a_j$, the algorithm maintains a Beta distribution $\Beta(a_{i,j},b_{i,j})$ for the preference value.
In the beginning, this distribution is initialized to $\Beta(1,1)$, the uniform distribution on $[0,1]$ (line \ref{alg:cats:initial}). 
It will be later updated based on observed feedback and tend to concentrate on the real value $\mu_{i,j}$. 
In round $t$, the algorithm samples an index $\theta_{i,j}(t)$ from $\Beta(a_{i,j},b_{i,j})$ to represent the current estimation (line \ref{alg:cats:sample}).

One may consider letting each player $p_i$ follow independent TS strategies, namely choose $A_i(t) \in \argmax_{j \in [K]}\theta_{i,j}(t)$. 
A simple example shows that frequent conflicts would happen under this strategy. 
Suppose there are $2$ players and $2$ arms, and $\mu_{i,1} = \max_{j\in[K]}\mu_{i,j}$ for each $p_i$. 
Then at each round $t$, $\theta_{i,1}(t) = \max_{j\in[K]}\theta_{i,j}(t)$ holds for each $p_i$ with at least constant probability. 
The reason is that if enough observations are collected, then the sampled indices approach the real preference value and $\theta_{i,1}(t)$ would be the largest; otherwise, the Beta distributions tend to be uniform and $\theta_{i,1}(t)$ would still be the largest with constant probability. 
Thus both players attempt to pull $a_1$ with constant probability at each round.
However, $a_1$ only accepts one player and the other one will be rejected and receive no feedback. The stable regret of the latter player is thus of order $O(T)$.

To avoid frequent conflicts in the above case, we construct a plausible set for each player to exclude arms that may reject it (line \ref{alg:cats:plausible}). 
Recall that the successfully matched players and the preference rankings of arms are known. 
If arm $a_j$ accepts a player $p_{i'}$ but prefers $p_i$ more, $p_i$ will not be rejected by $a_j$ when the same set of players attempt to pull $a_j$. Following this observation, the plausible set for $p_i$ at time $t$ is constructed to contain arms who accept a player less preferred than $p_i$ at $t-1$ (line \ref{alg:cats:plausible}). Player $p_i$ then selects the arm with the highest index in the plausible set (line \ref{alg:cats:select:end}).
Here we use $\bar{A}_{i}(t)$ to represent the arm successfully pulled by $p_{i}$ at $t$. If $p_i$ fails the conflict, then $\bar{A}_{i}(t)=-1$. 

However, players can still simultaneously pull same arms. 
% by selecting arms in the plausible set (line \ref{alg:cats:plausible}-\ref{alg:cats:select:end}). 
Consider an example with $2$ players and $2$ arms. Both players have $\mu_{i,1} = \argmax_{j\in[K]}\mu_{i,j}$, arm $a_1$ prefers $p_1$ to $p_2$ and arm $a_2$ prefers $p_2$ to $p_1$. It is possible that both players attempt to pull $a_1$ at $t=1$. Then $p_2$ is rejected and its plausible set only contains $a_2$ at $t=2$. For player $p_1$, its plausible set is still $\set{a_1,a_2}$, and its Beta distribution for arm $a_2$ is still uniform. Thus with constant probability, $\theta_{1,2}(t)>\theta_{1,1}(t)$ and both players attempt to pull arm $a_2$ at time $t=2$. 
In this case, player $p_1$ will then be rejected. The same analysis shows that both players can always pull the same arm and be rejected in turn. Thus the stable regrets of these two are both of the order $O(T)$.

A random delay mechanism is further introduced to keep the effectiveness of the plausible set and avoid such rejections. 
The intuition is that if all players except $p_i$ follow the last-round choice, then $p_i$ will not be rejected by selecting arms in the plausible set. 
To be specific, each player first draws a Bernoulli random variable $D_i(t)$ with expectation $\lambda$ (line \ref{alg:cats:draw}), which is a hyper-parameter. 
If $D_i(t)=0$, $p_i$ still attempts the arm with the largest index in the plausible set (line \ref{alg:cats:select:start}-\ref{alg:cats:select:end}); otherwise, it follows the last-round choice (line \ref{alg:cats:delay}). 

% The delay parameter $\lambda$ (line \ref{alg:cats:input}) is used to determine whether to follow the choice in the last round. 

When all players decide which arm to pull in this round, arms will determine which player to accept according to their rankings. If player $p_i$ wins the conflict and successfully pulls arm $A_i(t)$, the algorithm marks $\bar{A}_i(t)$ as $A_i(t)$ and updates the corresponding Beta distribution (line \ref{alg:cats:update:start}-\ref{alg:cats:update:end}).

%!TEX root =  main.tex

\section{Theoretical Results}
\label{sec:analysis}

In this section, we provide the theoretical result of CA-TS. 
The corresponding gaps are defined to measure the performance of the algorithm. 

\begin{definition}{(Gaps)}
% \shuai{formular correct?}
For each player $p_i \in \cN$, denote $\Delta_{i,\max} = \mu_{i,m_i}$
% \begin{align}
% \Delta_{i,\max} = \max_{j \in [K] }\set{ \mu_{i,{m}_i} - \mu_{i,j}, \mu_{i,{m}_i}}
% \end{align}
as the maximum stable regret that player $p_i$ needs to pay in unstable matchings. 
For any pair of arms $a_j$ and $a_{j'}$, define
\begin{align}
    \Delta_{i,j,j'} = \abs{\mu_{i,j} - \mu_{i,j'}}
\end{align}
as the reward difference between arm $a_j$ and $a_{j'}$ for player $p_i$. Let $\Delta = \min_{i,j,j':\mu_{i,j}\neq \mu_{i,j'} } \abs{\mu_{i,j}-\mu_{i,j'}}$.

\end{definition}

% Denote $M^* = \set{ m\mid m:\cN\to \cK, m \text{ is stable} }$ as the set of all stable matchings between the player set $\cN$ and the arm set $\cK$. We then provide the stable regret upper bound for CA-TS in Theorem \ref{thm:bound}. 
\begin{theorem}\label{thm:bound}
Let $\rho = \lambda^{N-1}(1-\lambda)$. Following Algorithm \ref{alg:CA-TS}, the stable regret of each player $p_i$ satisfies

\begin{align}
R_i(T) 
% \le& \sum_{t=1}^T \PP{A(t) \notin M^*} \Delta_{i,\max} \notag\\
\le&   \left\{ \frac{2N^5K^2\log T}{\rho^{N^4}} \bracket{ 4+\frac{6\log T}{\rho (\Delta-2\varepsilon)^2} + \frac{C}{\rho\varepsilon^6} }  \right.\notag \\
&\left.~~ + 6 + \frac{6N^4}{\rho^{N^4}} \right\}\cdot \Delta_{i,\max} \label{bound:ful}\\
=&O\bracket{ \frac{N^5K^2 \log^2 T}{\rho^{N^4}\Delta^2}\cdot \Delta_{i,\max} } \label{bound:order}\,,
\end{align}
for any $\varepsilon$ such that $\Delta - 2\varepsilon >0$, where $C$ is a universal constant.
\end{theorem}
Due to the space limit, we provide
\ifsup
a proof sketch in Appendix \ref{sec:prooksketch} and the full proof in Appendix \ref{sec:proof:dcts}. 
\else
the proof sketch as well as the full proof in the Appendix. 
\fi
In addition to Beta priors shown in Algorithm \ref{alg:CA-TS}, we also analyze the CA-TS algorithm with Gaussian priors for $1$-subgaussian rewards, which achieves the same order of regret upper bound as Theorem \ref{thm:bound}. For completeness, we provide the full algorithm and analysis in
\ifsup
Appendix \ref{sec:CA-TS:gaussian}. 
\else
the Appendix.
\fi

\subsection{Discussions}\label{sec:discuss}
% \shuai{Discussions}% \shuai{Hardness}
\paragraph{Hardness}

The TS-type algorithm faces new challenge for analysis in the setting of matching markets. 
Note in this setting, once a player wrongly estimate arms, both this player and others in the market could suffer non-negative stable regret.
In the following, we take a player $p_i$ and two arms $a_j,a_{j'}$ with $\mu_{i,j'}>\mu_{i,j}$ for further analysis. 
Specifically, we need to bound the number of times when $p_i$ incorrectly attempts $a_j$ instead of $a_{j'}$. 
If the attempt is due to inaccurate estimations of $\mu_{i,j}$, then it can be easily guaranteed since with more selections, $\mu_{i,j}$ would be finally observed enough times and estimated well. However, if $\mu_{i,j}$ has already been estimated accurately, the analysis gets more complicated. 

% The standard TS analysis solve this problem by investigating the number of arms’ previous selections to constrain the variance of the posterior distribution and thus guarantee future good samples. However, such a technique is unable to deal with the difficulty in matching markets. This is because the selection of the ‘optimal’ arm does not imply a successful observation when the arm is influenced by other market participants. Thus the variance of the posterior distribution cannot be simply decided by the number of arms’ selections.

% To guarantee this part of regret, it needs to investigate the number of arms’ successful observations and relate this quantity with the number of arms’ selections. However, it is hard to find a unified relationship between these two quantities because the probability of getting an observation varies with different players at different times. To deal with this difficulty, we construct a new counter and update it only with a precisely calculated probability when a certain condition is satisfied. Relying on such updating rule, this counter can successfully lower bound the number of observations and also be linked with the number of selections. With this counter, the total horizon can be well separated into slices and the variance of posterior distribution within each slice can be bounded with the known lower bound for the number of observations. 

For this important part, we first investigate what properties of $\theta_{i,j'}$ are implied by the event. We show that the sample $\theta_{i,j'}$ must be inaccurate and once an accurate sample of it is drawn, $p_i$ will select $a_{j'}$. 
In standard analysis of MAB, such cases can be guaranteed since when $\theta_{i,j'}$ is not accurate, the variance of the posterior distribution would push it to generate a larger sample in some round and the observations can thus be obtained. 
But in matching markets, even when an accurate sample for $\mu_{i,j'}$ is drawn, the observation may still be unavailable. 
To deal with this case, at each time with a good sample of $\mu_{i,j'}$, we analyze the influence of all other players and compute the exact probability of obtaining an observation. 
Based on this probability, we construct a new counter to estimate the number of observations on $\mu_{i,j'}$. 
The total horizon can then have a fine division of slices based on this counter and within each slice, the posterior distribution of $\mu_{i,j'}$ has some common properties and we are able to bound the number of rounds to wait for a good event of being matched with arm $a_{j'}$.

Note this difficulty does not need to be dealt with by UCB and ETC. In UCB, due to the monotonicity between UCB index and real parameter, inaccurate estimations of $\mu_{i,j'}$ would not contribute to the regret, while ETC forces players to collect enough observations on every arm without considering other participants' influence.

\paragraph{Regret bound}
Our CA-TS is the first TS-type algorithm for general decentralized matching markets. 
Two comparable algorithms in the same setting with regret guarantees are CA-UCB \cite{liu2020bandit} and PhasedETC \cite{basu21beyond}. 
The former has the same main order of regret upper bound as ours. The latter, though its upper bound has better dependence on $T$ (of order $O(\log^{1+\varepsilon}T+2^{(1/\Delta^2)^{1/\varepsilon}})$, $\varepsilon>0$ is a hyper-parameter), suffers from the problem of cold-start and only works for a huge horizon $T=\Omega\bracket{\exp\bracket{N/\Delta^2}}$.
Compared with CA-UCB, our regret upper bound has an additional constant term $1/\varepsilon^6$. This term comes from the special nature of TS as discussed above and also appears in many other TS-based works \cite{agrawal2013further,wang2018thompson,perrault2020statistical,kong2021hardness}. 
As stated in the theorem, if we let $\varepsilon$ take value of $\Delta/3$, then this term would become $729/\Delta^6$, which is a constant relative to $\log T$. It is not known whether such a term is unavoidable and we would leave this as future work.
It is worth noting that the additional constant term does not imply bad practical performances. As shown in following Section \ref{sec:exp}, our TS-type algorithm performs better than these baselines.

\paragraph{Player-pessimal stable regret} 
In this paper, we study the player-pessimal stable regret which is defined with respect to players' least favorite stable matching. 
% Our objective is the minimize the player pessimal-stable regret, compared with the least favorite arm in all stable matchings.  
% Classic bandit algorithms face new challenges when they are used to solve the matching market problem. 
Except for basic ETC, the analyses of all existing works also focus on this type of stable regret \cite{liu2020competing,liu2020bandit,sankararaman2021dominate,basu21beyond}.
In the market shown in \cite[Example 6]{liu2020competing}, both UCB and TS cannot achieve sublinear player-optimal stable regret, compared with players' favorite stable matching. This is because once a player mistakenly ranks two arms, other players’ behavior can force it to have no chance for more observations and learn a correct preference ranking for optimal stable matching. How to minimize the player-optimal stable regret in general matching markets is still an open problem. 

% \shuai{they do not give TS. For the example shown in xxx, both xxx }

%!TEX root =  main.tex

\section{Experiments}
\label{sec:exp}

In this section, we compare the performances of our CA-TS with other related baselines in different environments\footnote{The code is available at \url{https://github.com/fangkongx/TSforMatchingMarkets}.} where all players and arms share the same preferences (Section \ref{sec:exp:global}), the minimum reward gap $\Delta$ is varied (Section \ref{sec:exp:gap}), and the market size is varied (Section \ref{sec:exp:size}). 
The baselines include CA-UCB \cite{liu2020bandit}, PhasedETC (P-ETC) \cite{basu21beyond}, Decentralized ETC (D-ETC) \cite{liu2020competing} and UCB-D4 \cite{basu21beyond}. 
Since UCB-D4 requires the market to satisfy uniqueness consistency, we only test it in global-preference case where only unique stable matching exists (Section \ref{sec:exp:global}). 
The hyper-parameters of all algorithms are set as their original paper with details in 
\ifsup
Appendix \ref{sec:add:exp}. 
\else
the Appendix.
\fi

% The parameter of D-ETC is not mentioned in their paper and we will report our choice later.

% Due to the space limit, the introduction to these algorithms and the selections of hyper-parameters are deferred to Appendix \ref{sec:baselines}. 

We first report the cumulative stable regrets for all experiments. Also since the goal is to learn stable matchings and is not well reflected from cumulative regrets, we also report the cumulative market unstability, which is defined as the number of unstable matchings over $T$ rounds, as in \citeauthor{liu2020bandit} \shortcite{liu2020bandit}. 
In each market, we run all algorithms for $T=100k$ rounds and all results are averaged over $50$ independent runs. The error bars correspond to standard errors, which are computed as standard deviations divided by $\sqrt{50}$.

% In the first three experiments, the reward $X_{i,j}(t)$ for player $p_i$ at time $t$ is drawn independently from $\mathrm{Bernoulli}(\mu_{i,j})$ if $p_i$ successfully pulls arm $a_j$. 
% In the last one, to 
% Both Bernoulli and Gaussian-type rewards are tested. 

\subsection{Global Preferences}\label{sec:exp:global}

%  and test the performances of CA-TS, CA-UCB, P-ETC, D-ETC, and UCB-D4
In this section, we construct a market of $N=5$ players and $K=5$ arms with global preferences, where all players share the same preference over arms and all arms share the same preference over players. 
Specifically, we set $\mu_{i,1}>\mu_{i,2}>\mu_{i,3}>\mu_{i,4}>\mu_{i,5}$ for each player $p_i$ and $\pi_{j,1}> \pi_{j,2}> \pi_{j,3}> \pi_{j,4}> \pi_{j,5}$ for each arm $a_j$. The preference towards the least favorite arm is set as $\mu_{i,5}=0.1$ for any $p_i$ and the reward gap between any two consecutively ranked arms is set as $\Delta=0.2$.
In this market, the unique stable matching is $\set{(p_1,a_1),(p_2,a_2),(p_3,a_3),(p_4,a_4),(p_5,a_5)}$, thus the assumption of uniqueness consistency required by UCB-D4 is satisfied. 
The feedback $X_{i,j}(t)$ for player $p_i$ on arm $a_j$ at time $t$ is drawn independently from Bernoulli$(\mu_{i,j})$. 
% in this and the following two experiments.  

We report the cumulative stable regret of each player in Figure \ref{fig:global} (a-e) and the cumulative market unstability for each algorithm in Figure \ref{fig:global} (f).

\begin{figure}[th!] 
\includegraphics[width=0.49\linewidth]{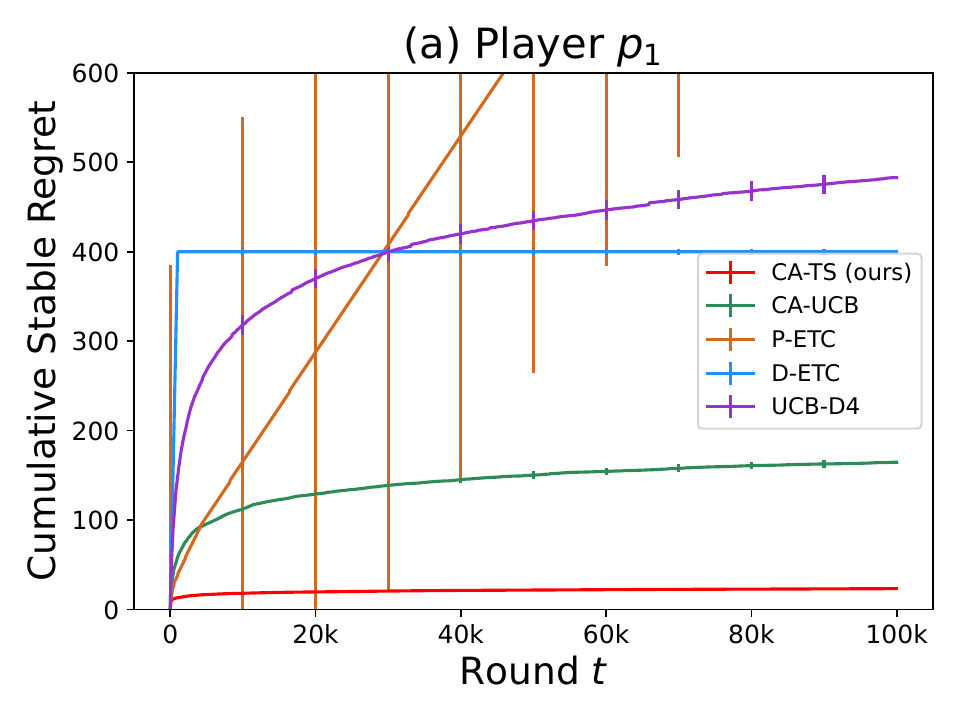}
\includegraphics[width=0.49\linewidth]{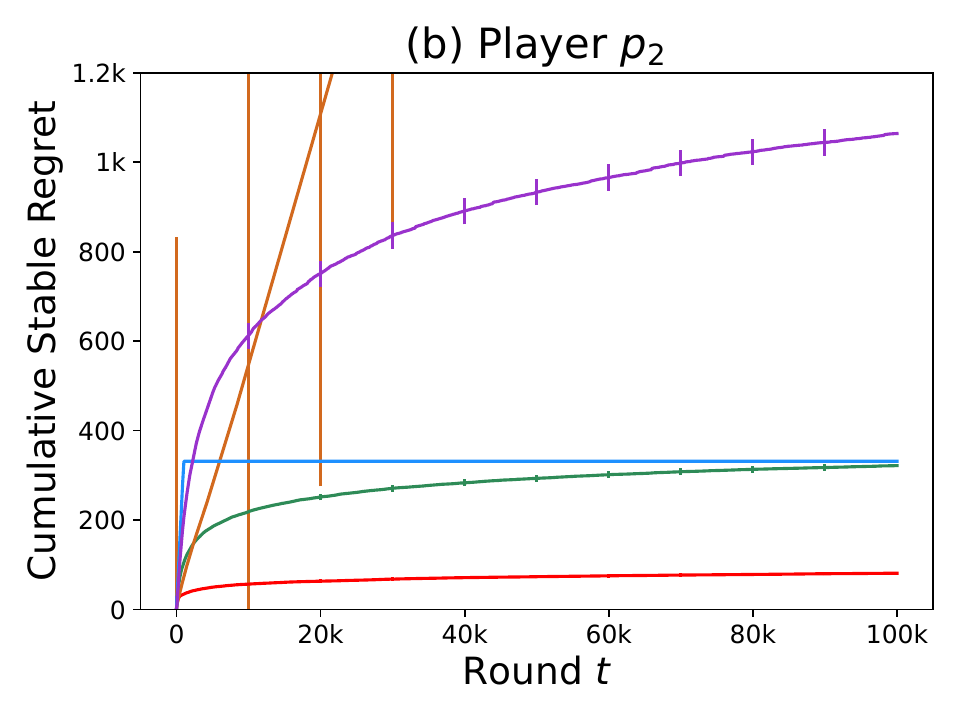} \\
\includegraphics[width=0.49\linewidth]{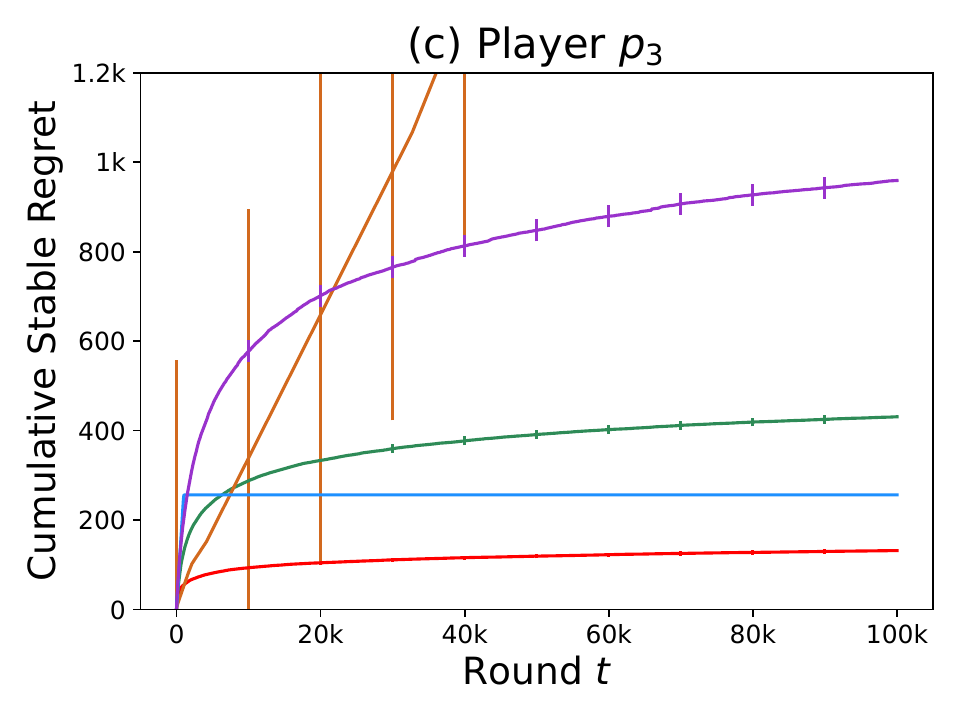}
\includegraphics[width=0.49\linewidth]{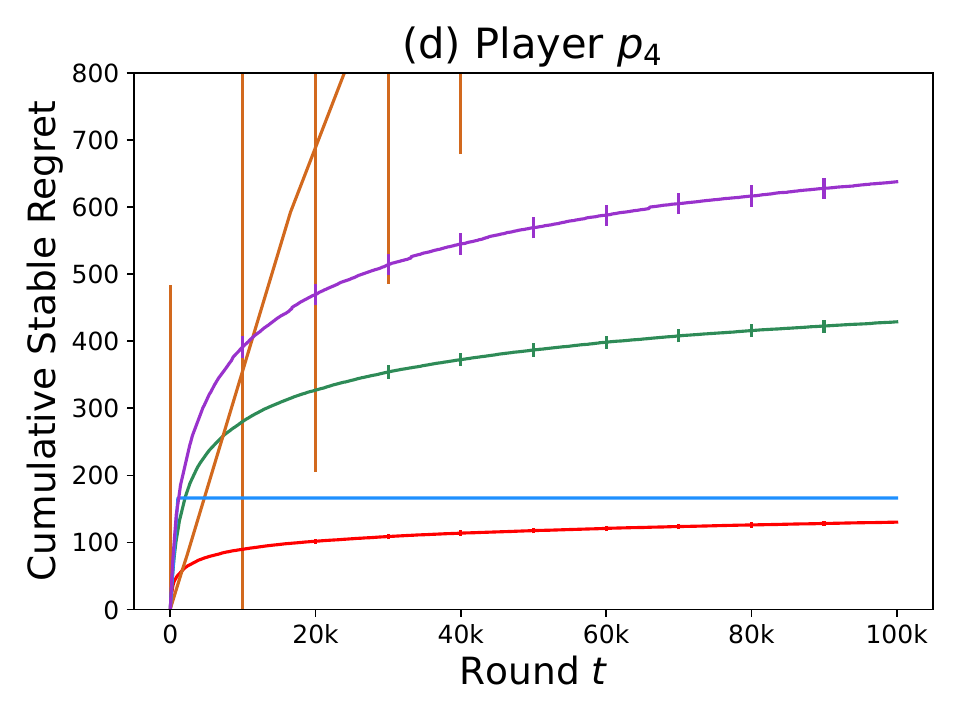}\\
\includegraphics[width=0.49\linewidth]{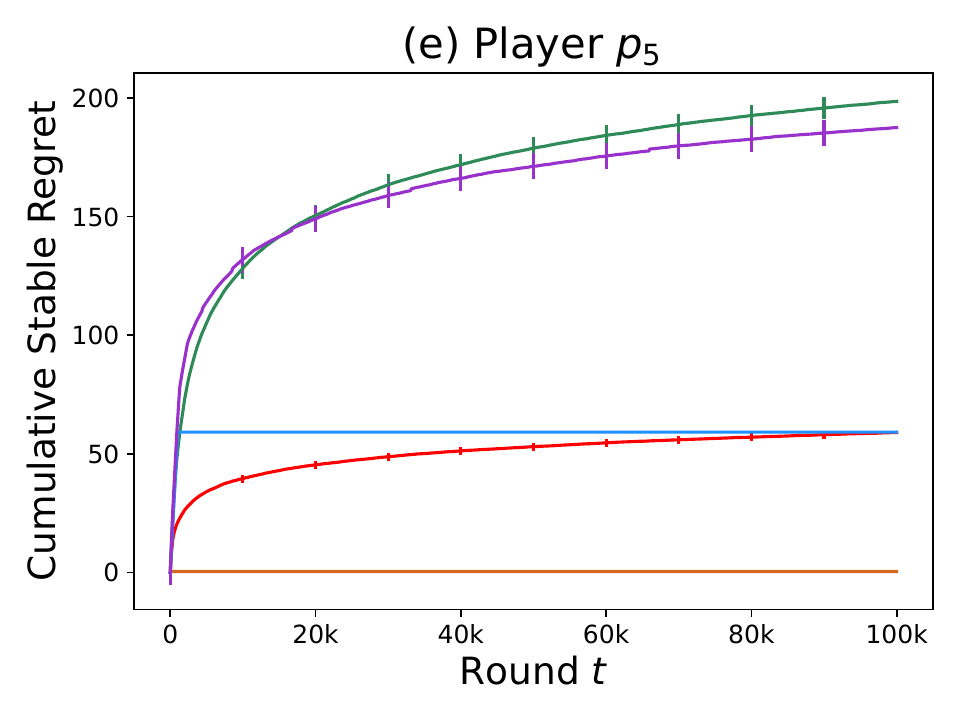}
\includegraphics[width=0.49\linewidth]{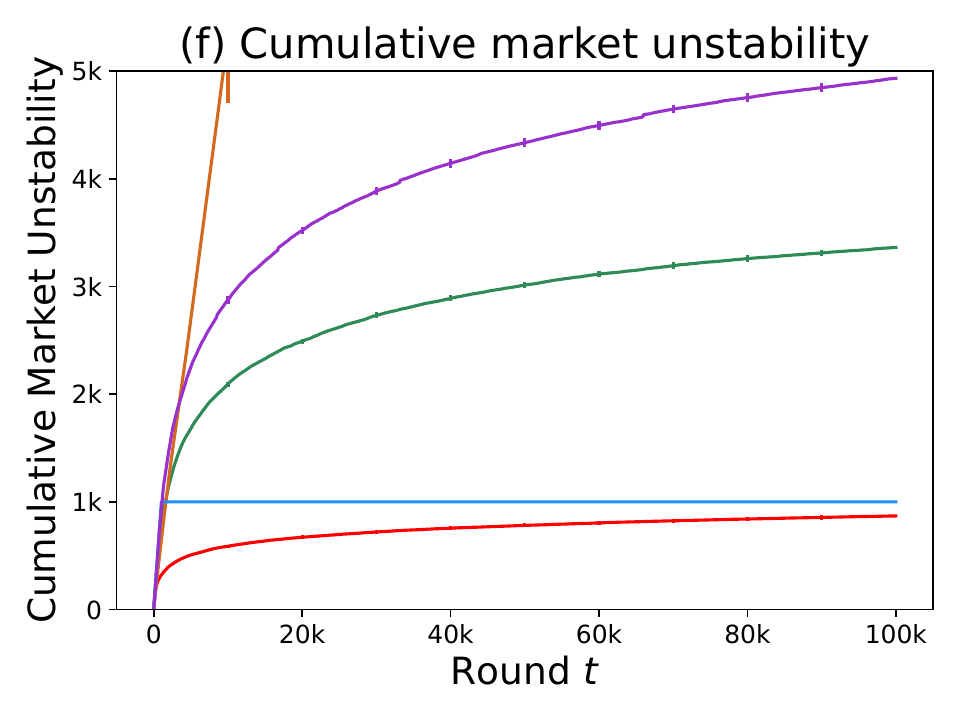}
  \caption{Experimental comparisons of our CA-TS with CA-UCB, P-ETC, D-ETC, and UCB-D4 in the market of $N=5$ players and $K=5$ arms with global preferences.}
  \label{fig:global}
\end{figure}

Our CA-TS pays the least stable regret on player $p_1,p_2,p_3$, and $p_4$ among all algorithms, while only pays higher regret than P-ETC on player $p_5$. 
However, under P-ETC, both the market unstability and the stable regrets of other players still do not converge, which phenomenon also coincides with the theoretical result that P-ETC suffers from cold-start and only works when $T$ is huge \cite{basu21beyond}. 
The reason for lower stable regret of $p_5$ is that P-ETC avoids rejections by following allocations of the GS algorithm and $p_5$ would suffer zero regret in this process since all other arms are better than its partner $a_5$ in the stable matching. 
For other algorithms, the stable regrets are not consistent among different players since they explore differently to find a stable matching. 
So we look into market unstability for further analysis. 
As shown in Figure \ref{fig:global} (f), our CA-TS shows the least market unstability. 
By carefully choosing an appropriate hyper-parameter $H$ representing the exploration budget, D-ETC performs slightly worse than ours. 
CA-UCB and UCB-D4 converge much slower and explore more to find a stable matching.

\subsection{Varying Gaps for Random Preferences}\label{sec:exp:gap}

In this experiment, we test the effect of the minimum reward gap $\Delta$ on the stable regret and market unstability. 
The market size is fixed with $N=5$ players and $K=5$ arms. The preference rankings of all players (arms) are generated as random permutations of arms (players, respectively).
We set the preference value towards the least favorite arm as $0.1$ for all players and test four different choices of reward gap between any two consecutively ranked arms 
$\Delta\in \set{0.2,0.15,0.1,0.05}$.  

Figure \ref{fig:delta} shows the maximum cumulative stable regret among all players and the cumulative market unstability for each algorithm under different values of $\Delta$. 

\begin{figure}[th!] 
\includegraphics[width=1\linewidth]{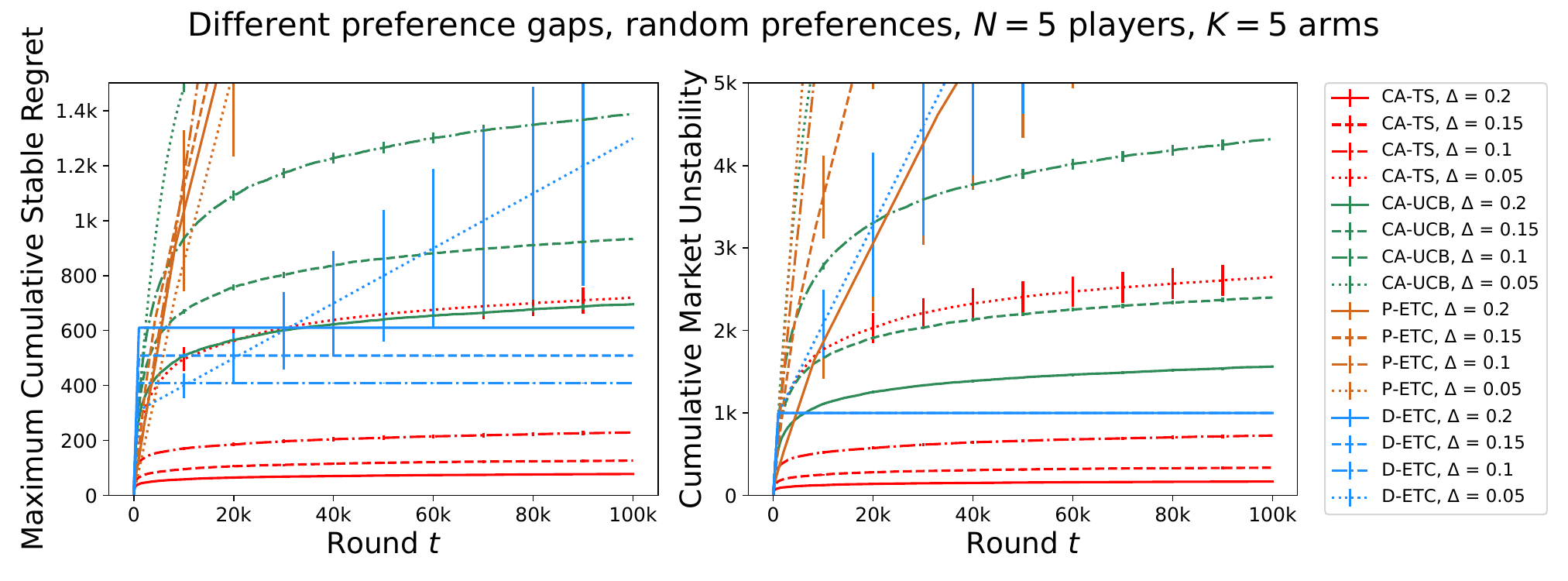}
  \caption{Experimental comparisons of our CA-TS with CA-UCB, P-ETC, and D-ETC in terms of maximum cumulative stable regret among players (left) and cumulative market unstability (right). Markets of $N=5$ players and $K=5$ arms with different $\Delta$ are tested.
   % The market size is set as $N=5$ players, $K=5$ arms.
  }
  \label{fig:delta}
\end{figure}

Our CA-TS pays the least stable regret among algorithms under each value of $\Delta$ and performs the most robust when $\Delta$ is changed.  
The D-ETC performs second to ours when $\Delta\in \set{0.2,0.15,0.1}$. However, its regret does not converge when $\Delta=0.05$, which illustrates its high dependence on the selection of $H$ and confirms the theoretical result that $H$ needs to depend on $\Delta$ to guarantee good performances \cite{liu2020competing}.
The CA-UCB algorithm pays higher stable regret than our CA-TS and its performances also change more drastically with different $\Delta$. 
The P-ETC algorithm still suffers from cold-start and does not converge in given periods. 

Similar observations can also be found from the perspective of the market unstability. As shown in Figure \ref{fig:delta} (right), the cumulative market unstability of algorithms gets higher when $\Delta$ is small. This is as expected since algorithms need to explore more to get accurate estimations on preference rankings and thus find the true stable matching. 

One may also concern that the performances of algorithms could be worse when the preferences of all players tend to be the same since players may always attempt same arms thus leading to more conflicts. In 
\ifsup
Appendix \ref{sec:add:exp}, 
\else
the Appendix, 
\fi
we test the performances of algorithms in markets with different heterogeneity of players' preferences and find our CA-TS still shows the best and the most robust performances.

% also follow the experimental setting of \citeauthor{liu2020bandit} \shortcite{liu2020bandit} to test the performances of algorithms in markets with different heterogeneity of players' preferences. \fang{the heterogeneity of players' preferences is varied (Section \ref{sec:exp:beta})}

\subsection{Varying Market Size for Random Preferences}\label{sec:exp:size}

In this section, we investigate how performances of algorithms are influenced by market sizes. 
Four markets with size $K=N\in \set{5,10,20,40}$ are tested. 
The preference rankings for players (arms) are generated as random permutations of arms (players, respectively). 
As in previous experiments, the preference value for the least favorite arm is set as $0.1$ and the reward gap between any two consecutively ranked arms is set as $\Delta=0.2$.
The Bernoulli-type reward thus is unsuitable since the preference value may be larger than $1$. Here we adopt the Gaussian-type reward with unit variance and run CA-TS with Gaussian priors
\ifsup
(Algorithm \ref{alg:CA-TS-gaussian} in Appendix \ref{sec:CA-TS:gaussian}). 
\else
. 
\fi

%  We set the preference value for the least favorite arm as $0.1$ and the reward gap between any two consecutively ranked arms as $\Delta=0.2$. In this case, the Bernoulli-type feedback is not suitable since the preference value may be larger than $1$. Here we adopt the Gaussian reward with variance $1$ and run CA-TS with Gaussian priors (Algorithm \ref{alg:CA-TS-gaussian}).

%  To ensure the same preference gap $\Delta=0.2$ in all markets, the reward is sampled from the Gaussian distribution with variance $1$.

% To ensure that each preference value $\mu_{i,j}$ falls in the interval $[0,1]$, we set the preference value for the least favorite arm as $0$ and the reward gap between any two consecutively ranked arms as $\Delta=0.05$. 

The maximum cumulative stable regret among players and the cumulative market unstability of algorithms in different markets are shown in Figure \ref{fig:N}. 

% \begin{figure}[th!] 
% \includegraphics[width=1\linewidth]{figures/decentralized_varyN.pdf}
%   \caption{Experimental comparisons of our CA-TS with CA-UCB, P-ETC, and D-ETC in terms of maximum cumulative stable regret among players (left) and cumulative market unstability (right). 
%   Markets with different sizes $K=N \in \set{5,10,15,20}$ are tested.}
%   \label{fig:N}
% \end{figure}

\begin{figure}[th!] 
\includegraphics[width=1\linewidth]{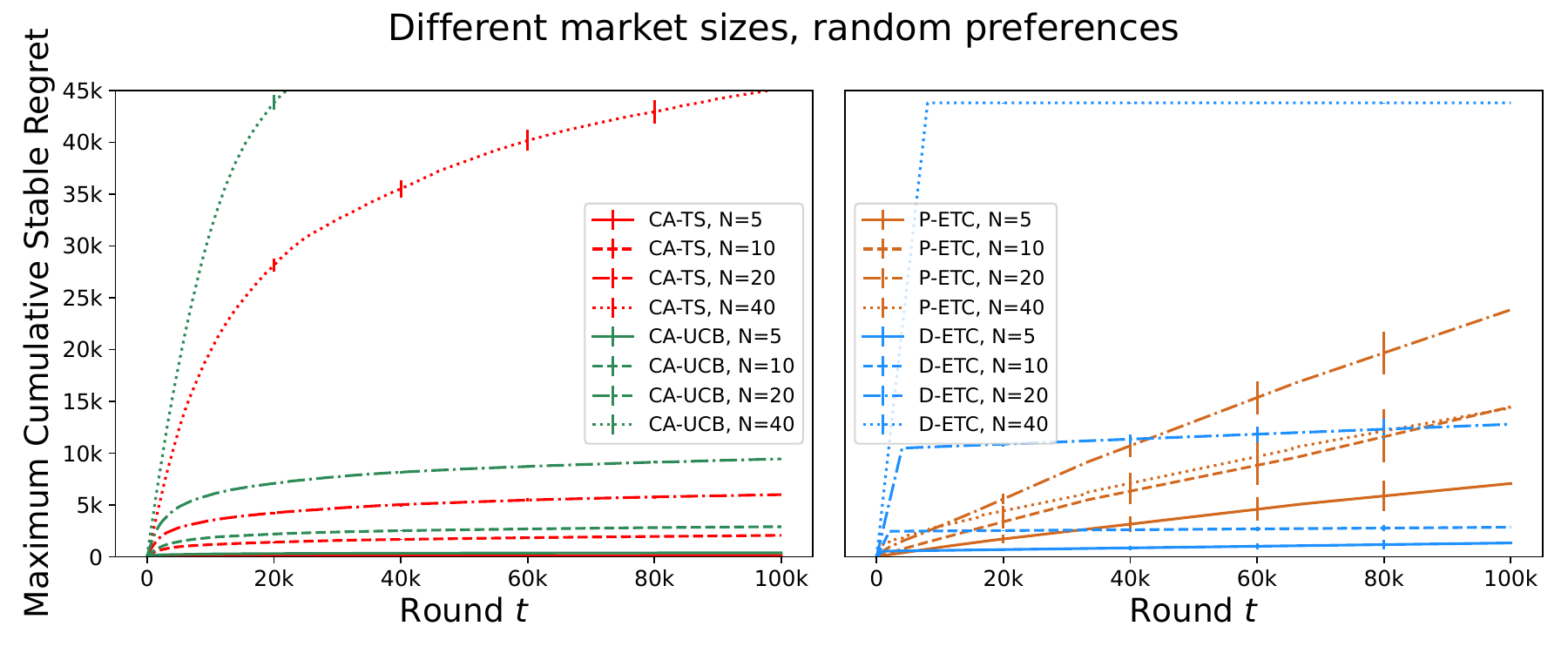}
\includegraphics[width=1\linewidth]{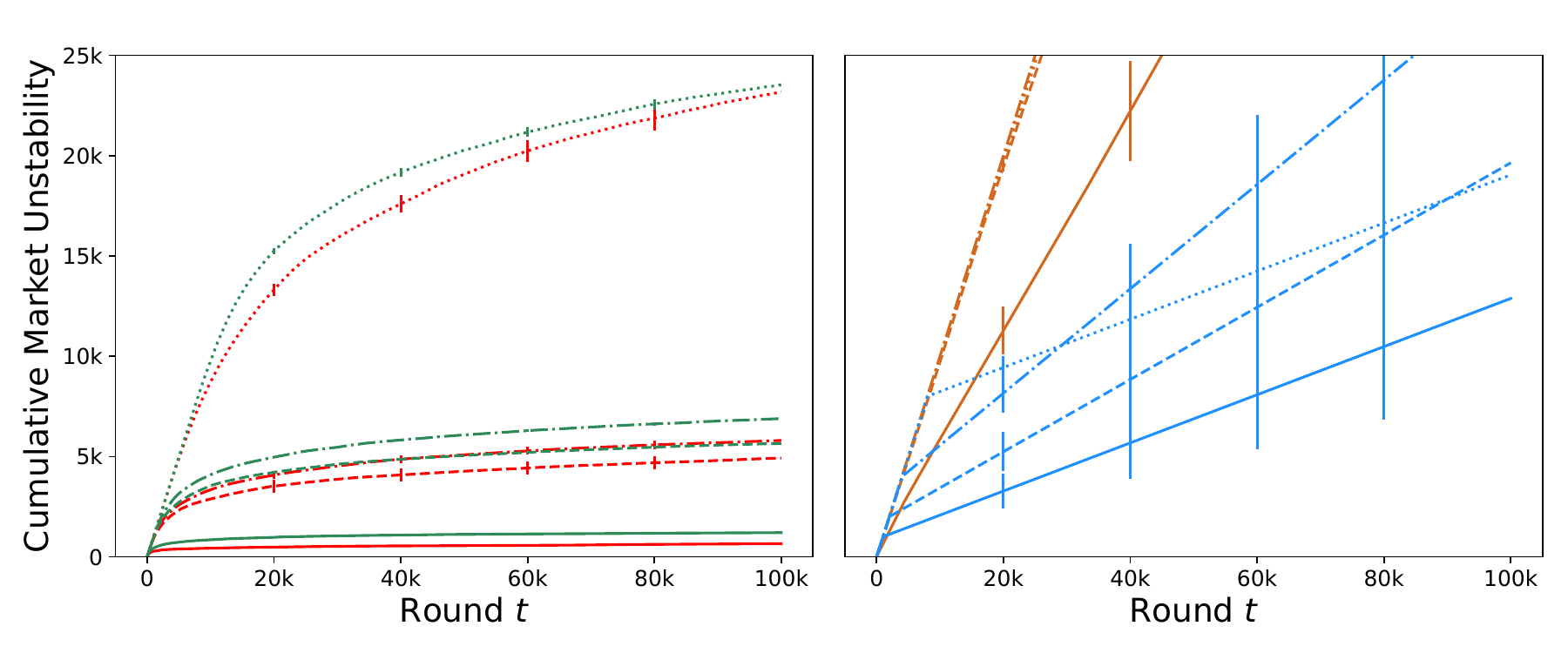}
  \caption{Experimental comparisons of our CA-TS with CA-UCB, P-ETC, and D-ETC in terms of maximum cumulative stable regret among players (top) and cumulative market unstability (bottom). 
  Markets with different sizes $K=N \in \set{5,10,20,40}$ are tested.}
  \label{fig:N}
\end{figure}

When adopting Gaussian-type reward, our CA-TS still shows the least stable regret among all algorithms when $N=5,10,20$, while is only slightly worse than D-ETC when $N=40$. But as shown in terms of market unstability, D-ETC still does not converge in these four markets. 
This is because some runs (among $50$) are non-convergent and shows growing unstability and regret. But they wrongly match players with closely ranked arms compared with the matched arm in the stable matching, thus the regret though increasing, is not large.  
Since we only report the maximum regret among all players, those non-convergent runs do not contribute to this metric in given periods. 
% \fang{
% The low stable regret of it comes from the reason that non-convergent runs just misplace the closely-ranked arms in unstable matchings and thus do not contribute to the metric of maximum stable regret among players. }
The CA-UCB algorithm performs a little worse than ours and the regret of P-ETC still does not converge and keeps increasing. 
Similar observations can also be found in terms of market unstability. 
As expected, algorithms converge more slowly and explore more in larger market.
Though our theoretical result (Theorem \ref{thm:bound}) shows exponential dependence on $N$, this dependence demonstrated by experimental results is much better.

% \fang{ Rewrite 

% % Since the small gap $\Delta=0.05$ makes the stable regrets of algorithms increase very slowly and difficult to be distinguished, we would like to first look into the market unstability. 
% In each market, our CA-TS outperforms all baselines. It also shows robust performances for different market sizes. 
% The CA-UCB algorithm achieves much higher market unstability than CA-TS, and it is more sensitive to changes in market size. 
% The market unstability of both two ETC-type algorithms still does not converge and keeps increasing. 
% In general, the experimental results show algorithms converge more slowly in larger markets, which also coincides with the dependence on $N$ of theoretical results. 
% % \cite{liu2020competing,liu2020bandit,basu21beyond}.

% We then look into the stable regret. 
% Though our stable regret upper bound shows exponential dependence on $N$, 
% this dependence demonstrated by experimental results is much better than exponential. The performances of CA-UCB are very similar to those in terms of market unstability. 
% The D-ETC algorithm shows lower stable regret than ours when $N\in \set{10,15,20}$ and P-ETC shows lower stable regret when $N = 20$. But their stable regrets still slowly increase and do not converge. Combining their performances in terms of market unstability, we can conclude that their lower stable regrets may come from the reason that they just misplace the closely-ranked arms in unstable matchings. }

\section{Conclusion}\label{sec:conclusion}

In this paper, we show the first TS-type algorithm, CA-TS, for the two-sided decentralized matching markets. 
Compared with previous UCB and ETC-type algorithms, the special nature of TS and the setting of matching markets bring additional analysis difficulties, since TS requires enough observations on arms in stable matchings but market participants may force a player to observe no feedback.  
We overcome the difficulty by bounding the success probability of observing these arms with a fine division of the horizon based on the number of previous attempts and the analysis of other players' influence. 
Our stable regret upper bound achieves the same order as previous UCB-type algorithms. 
Extensive experiments on markets with different properties are conducted to verify the practical performances of the TS-type algorithm. Compared with all baselines, our CA-TS learns stable matchings much faster and shows more robust performances. 
% \shuai{sentence stop here}when the properties of the market change. 

% We also conduct a series of experiments by varying different properties of the market to verify the practical performances of the algorithm. 
% Compared with all baselines, our CA-TS learns stable matchings much faster and shows more robust performances when the properties of the market are changed. 

% We overcome the analysis difficulty brought by the special nature of TS by analyzing the properties of the parameters sampled from posterior distributions at different times.  
% Our stable regret upper bound for CA-TS achieves the same order as the previous UCB-type algorithm. 
% We also conduct a series of experiments by varying different properties of the market to verify the practical performances of the algorithm. 
% Compared with all baselines, our CA-TS learns stable matchings much faster and shows more robust performances when the properties of the market are changed. 

We study a decentralized setting where players independently choose arms. An interesting future direction is to design TS-type algorithms for scenarios with more  communication such as the centralized setting in \citeauthor{liu2020competing}\shortcite{liu2020competing} where players can communicate with a central platform.  
New design ideas are needed as the requirement of TS is more restrictive than UCB. 
An example in 
\ifsup
Appendix \ref{sec:conterexample}
\else
the Appendix
\fi 
shows that if we only replace the UCB index in \citeauthor{liu2020competing} \shortcite{liu2020competing} with the sampled parameters from posterior distributions as TS, players can suffer a stable regret of order $O(T)$. 
This example further illustrates the difference between TS and UCB and the challenge of TS analysis in matching markets.

% This example further illustrates the challenge of designing a TS-type algorithm for the centralized setting. \shuai{this statement is weird since you can always run a decentralized algorithm in centralized setting, just ignore the platform}

% \section*{Acknowledgement}
% The corresponding author Shuai Li is supported by National Natural Science Foundation of China (62006151, 62076161). This work is sponsored by Shanghai Sailing Program.

% \newpage

\bibliographystyle{named}
\bibliography{ref}

\begin{thebibliography}{}

\bibitem[\protect\citeauthoryear{Abdulkadiro{\u{g}}lu and
  S{\"o}nmez}{1999}]{abdulkadirouglu1999house}
Atila Abdulkadiro{\u{g}}lu and Tayfun S{\"o}nmez.
\newblock House allocation with existing tenants.
\newblock {\em Journal of Economic Theory}, 88(2):233--260, 1999.

\bibitem[\protect\citeauthoryear{Abramowitz and
  Stegun}{1964}]{abramowitz1964handbook}
Milton Abramowitz and Irene~A Stegun.
\newblock {\em Handbook of Mathematical Functions with Formulas, Graphs, and
  Mathematical Tables}.
\newblock 1964.

\bibitem[\protect\citeauthoryear{Agrawal and Goyal}{2012}]{agrawal2012analysis}
Shipra Agrawal and Navin Goyal.
\newblock Analysis of thompson sampling for the multi-armed bandit problem.
\newblock In {\em Conference on Learning Theory}, pages 39--1, 2012.

\bibitem[\protect\citeauthoryear{Agrawal and Goyal}{2013}]{agrawal2013further}
Shipra Agrawal and Navin Goyal.
\newblock Further optimal regret bounds for thompson sampling.
\newblock In {\em International Conference on Artificial Intelligence and
  Statistics}, pages 99--107, 2013.

\bibitem[\protect\citeauthoryear{Ashlagi \bgroup \em et al.\egroup
  }{2017}]{ashlagi2017communication}
Itai Ashlagi, Mark Braverman, Yash Kanoria, and Peng Shi.
\newblock Communication requirements and informative signaling in matching
  markets.
\newblock In {\em ACM Conference on Economics and Computation}, pages 263--263,
  2017.

\bibitem[\protect\citeauthoryear{Auer \bgroup \em et al.\egroup
  }{2002}]{auer2002finite}
Peter Auer, Nicolo Cesa-Bianchi, and Paul Fischer.
\newblock Finite-time analysis of the multiarmed bandit problem.
\newblock {\em Machine Learning}, 47(2):235--256, 2002.

\bibitem[\protect\citeauthoryear{Basu \bgroup \em et al.\egroup
  }{2021}]{basu21beyond}
Soumya Basu, Karthik~Abinav Sankararaman, and Abishek Sankararaman.
\newblock Beyond $\log^2(t)$ regret for decentralized bandits in matching
  markets.
\newblock In {\em International Conference on Machine Learning}, pages
  705--715, 2021.

\bibitem[\protect\citeauthoryear{Bistritz and
  Leshem}{2018}]{bistritz2018distributed}
Ilai Bistritz and Amir Leshem.
\newblock Distributed multi-player bandits-a game of thrones approach.
\newblock In {\em Advances in Neural Information Processing Systems}, pages
  7222--7232, 2018.

\bibitem[\protect\citeauthoryear{Chen \bgroup \em et al.\egroup
  }{2020}]{chen2020combinatorial}
Wei Chen, Yihan Du, Longbo Huang, and Haoyu Zhao.
\newblock Combinatorial pure exploration for dueling bandit.
\newblock In {\em International Conference on Machine Learning}, pages
  1531--1541, 2020.

\bibitem[\protect\citeauthoryear{Cheung \bgroup \em et al.\egroup
  }{2019}]{cheung2019thompson}
Wang~Chi Cheung, Vincent Tan, and Zixin Zhong.
\newblock A thompson sampling algorithm for cascading bandits.
\newblock In {\em International Conference on Artificial Intelligence and
  Statistics}, pages 438--447. PMLR, 2019.

\bibitem[\protect\citeauthoryear{Dai and Jordan}{2021}]{dai2020learning}
Xiaowu Dai and Michael~I Jordan.
\newblock Learning strategies in decentralized matching markets under uncertain
  preferences.
\newblock {\em Journal of Machine Learning Research}, 22(260):1--50, 2021.

\bibitem[\protect\citeauthoryear{Das and Kamenica}{2005}]{das2005two}
Sanmay Das and Emir Kamenica.
\newblock Two-sided bandits and the dating market.
\newblock In {\em International Joint Conference on Artificial Intelligence},
  pages 947--952, 2005.

\bibitem[\protect\citeauthoryear{Epple \bgroup \em et al.\egroup
  }{2006}]{epple2006admission}
Dennis Epple, Richard Romano, and Holger Sieg.
\newblock Admission, tuition, and financial aid policies in the market for
  higher education.
\newblock {\em Econometrica}, 74(4):885--928, 2006.

\bibitem[\protect\citeauthoryear{Fu}{2014}]{fu2014equilibrium}
Chao Fu.
\newblock Equilibrium tuition, applications, admissions, and enrollment in the
  college market.
\newblock {\em Journal of Political Economy}, 122(2):225--281, 2014.

\bibitem[\protect\citeauthoryear{Gale and Shapley}{1962}]{gale1962college}
David Gale and Lloyd~S Shapley.
\newblock College admissions and the stability of marriage.
\newblock {\em The American Mathematical Monthly}, 69(1):9--15, 1962.

\bibitem[\protect\citeauthoryear{Haeringer and
  Wooders}{2011}]{haeringer2011decentralized}
Guillaume Haeringer and Myrna Wooders.
\newblock Decentralized job matching.
\newblock {\em International Journal of Game Theory}, 40(1):1--28, 2011.

\bibitem[\protect\citeauthoryear{Kaufmann \bgroup \em et al.\egroup
  }{2012}]{kaufmann2012thompson}
Emilie Kaufmann, Nathaniel Korda, and R{\'e}mi Munos.
\newblock Thompson sampling: An asymptotically optimal finite-time analysis.
\newblock In {\em International Conference on Algorithmic Learning Theory},
  pages 199--213. Springer, 2012.

\bibitem[\protect\citeauthoryear{Kong \bgroup \em et al.\egroup
  }{2021}]{kong2021hardness}
Fang Kong, Yueran Yang, Wei Chen, and Shuai Li.
\newblock The hardness analysis of thompson sampling for combinatorial
  semi-bandits with greedy oracle.
\newblock In {\em Advances in Neural Information Processing Systems},
  volume~34, 2021.

\bibitem[\protect\citeauthoryear{Lattimore and
  Szepesv{\'a}ri}{2020}]{lattimore2020bandit}
Tor Lattimore and Csaba Szepesv{\'a}ri.
\newblock {\em Bandit Algorithms}.
\newblock Cambridge University Press, 2020.

\bibitem[\protect\citeauthoryear{Liu and Zhao}{2010}]{liu2010distributed}
Keqin Liu and Qing Zhao.
\newblock Distributed learning in multi-armed bandit with multiple players.
\newblock {\em IEEE Transactions on Signal Processing}, 58(11):5667--5681,
  2010.

\bibitem[\protect\citeauthoryear{Liu \bgroup \em et al.\egroup
  }{2020}]{liu2020competing}
Lydia~T Liu, Horia Mania, and Michael Jordan.
\newblock Competing bandits in matching markets.
\newblock In {\em International Conference on Artificial Intelligence and
  Statistics}, pages 1618--1628. PMLR, 2020.

\bibitem[\protect\citeauthoryear{Liu \bgroup \em et al.\egroup
  }{2021}]{liu2020bandit}
Lydia~T Liu, Feng Ruan, Horia Mania, and Michael~I Jordan.
\newblock Bandit learning in decentralized matching markets.
\newblock {\em Journal of Machine Learning Research}, 22(211):1--34, 2021.

\bibitem[\protect\citeauthoryear{Perrault \bgroup \em et al.\egroup
  }{2020}]{perrault2020statistical}
Pierre Perrault, Etienne Boursier, Vianney Perchet, and Michal Valko.
\newblock Statistical efficiency of thompson sampling for combinatorial
  semi-bandits.
\newblock In {\em Advances in Neural Information Processing Systems}, 2020.

\bibitem[\protect\citeauthoryear{Rosenski \bgroup \em et al.\egroup
  }{2016}]{rosenski2016multi}
Jonathan Rosenski, Ohad Shamir, and Liran Szlak.
\newblock Multi-player bandits--a musical chairs approach.
\newblock In {\em International Conference on Machine Learning}, pages
  155--163. PMLR, 2016.

\bibitem[\protect\citeauthoryear{Roth and Sotomayor}{1992}]{roth1992two}
Alvin~E Roth and Marilda Sotomayor.
\newblock Two-sided matching.
\newblock {\em Handbook of Game Theory with Economic Applications}, 1:485--541,
  1992.

\bibitem[\protect\citeauthoryear{Roth and Xing}{1997}]{roth1997turnaround}
Alvin~E Roth and Xiaolin Xing.
\newblock Turnaround time and bottlenecks in market clearing: Decentralized
  matching in the market for clinical psychologists.
\newblock {\em Journal of Political Economy}, 105(2):284--329, 1997.

\bibitem[\protect\citeauthoryear{Roth}{1984}]{roth1984evolution}
Alvin~E Roth.
\newblock The evolution of the labor market for medical interns and residents:
  a case study in game theory.
\newblock {\em Journal of Political Economy}, 92(6):991--1016, 1984.

\bibitem[\protect\citeauthoryear{Sankararaman \bgroup \em et al.\egroup
  }{2021}]{sankararaman2021dominate}
Abishek Sankararaman, Soumya Basu, and Karthik~Abinav Sankararaman.
\newblock Dominate or delete: Decentralized competing bandits in serial
  dictatorship.
\newblock In {\em International Conference on Artificial Intelligence and
  Statistics}, pages 1252--1260. PMLR, 2021.

\bibitem[\protect\citeauthoryear{Thompson}{1933}]{thompson1933likelihood}
William~R Thompson.
\newblock On the likelihood that one unknown probability exceeds another in
  view of the evidence of two samples.
\newblock {\em Biometrika}, 25(3/4):285--294, 1933.

\bibitem[\protect\citeauthoryear{Wang and Chen}{2018}]{wang2018thompson}
Siwei Wang and Wei Chen.
\newblock Thompson sampling for combinatorial semi-bandits.
\newblock In {\em International Conference on Machine Learning}, pages
  5114--5122. PMLR, 2018.

\end{thebibliography}

\ifsup
\appendix
\onecolumn
\newpage
%!TEX root =  main.tex

\section{Proof Sketch of Theorem \ref{thm:bound}}\label{sec:prooksketch}
In this section, we provide a proof sketch of Theorem \ref{thm:bound}. Please see Appendix \ref{sec:proof:dcts} for the complete proof.

Denote $M^* = \set{ m\mid m:\cN\to \cK, m \text{ is stable} }$ as the set of all stable matchings between the player set $\cN$ and the arm set $\cK$.
The stable regret can be then upper bounded by 
\begin{align*}
R_i(T) \le \EE{\sum_{t=1}^T \bOne{A(t) \notin M^*}}\cdot \Delta_{i,\max}  \,.
\end{align*}
Next, we will analyze the event $\set{A(t) \notin M^*}$. 
Define 
\begin{align*}
	\cE_t = \bigcap_{i \in [N]} \set{\argmax_{j \in S_i(t)}\theta_{i,j}(t) = \argmax_{j \in S_i(t)}\mu_{i,j} }
\end{align*}
as the event that for any player $p_i \in \cN$, the arm $a_j$ with the highest index $\theta_{i,j}(t)$ is actually the player $p_i$'s favorite arm in the plausible set $S_i(t)$.  Lemma \ref{lem:generalize_preserve} shows a sufficient condition for the event $\set{A(t) \in M^*}$ to occur.

\begin{lemma}
\label{lem:generalize_preserve}
For any length $h_t=0,1,\ldots,t-2$, 
% \begin{align}
% \PP{A(t) \notin M^*}\le &\PP{\cap_{s=0}^{h_t} \bracket{\cE_{t-s} \cap \set{A(t-s-1) \notin M^*}}} \notag\\
% & + \PP{\cup_{s=0}^{h_t}\urcorner \cE_{t-s}  } \,. \label{eq:long_history}
% \end{align}
% \begin{small}
\begin{align*}
(\mathop{\cap}\limits_{s=0}^{h_t} \cE_{t-s}) \cap (\mathop{\cup}_{s=0}^{h_t} \set{A(t-s-1) \in M^* }) \subseteq \set{A(t) \in M^*}   \,. 
\end{align*}
% \end{small}
\end{lemma}

\begin{proof}
% We prove this result by showing that if $A(t) \notin M^*$, 
% $\cap_{s=0}^{h_t}   \bracket{E_{t-s}\cap \set{A(t-s-1)\notin M^*}}$ or $\cup_{s=0}^{h_t} \urcorner \cE_{t-s}$ holds. 

Recall in Algorithm \ref{alg:CA-TS}, $A_i(t)$ of each player $p_i$ can be the arm with the highest index in the plausible set or the arm $A_i(t-1)$ attempted by $p_i$ in the previous round. 
We first claim that if $A(t-1) \in M^*$ and $\cE_t$ holds, then $A(t)=A(t-1) \in M^*$. This claim can be proved by contradiction.  
% We put the formal description for it as well as the proof in Appendix \ref{sec:proof:dcts}. 

Suppose $A(t)\neq A(t-1)$, then there must exist a player $p_i$ such that $A_i(t) = a_j \neq A_i(t-1)$. According to Algorithm \ref{alg:CA-TS} and $A(t-1) \in M^*$, we have $a_j \in S_i(t)$ and $A_{i}(t-1) \in S_i(t)$. 
Based on event $\cE_t$, there must be $\mu_{i,j} > \mu_{i,A_{i}(t-1)}$, namely player $p_i$ prefers $a_j$ to $A_{i}(t-1)$. Let $p'$ be the player matched with $a_j$ at time $t-1$. Then there are two cases for $p'$.
If $p' = \emptyset$ and $a_j$ is matched with nobody at $t-1$, then in $A(t-1)$, both $p_i$ and $a_j$ prefer being matched with each other to current matched one. 
Otherwise, according to Line \ref{alg:cats:plausible} in Algorithm \ref{alg:CA-TS} and $a_j \in S_i(t)$, we can also conclude $p_i$ and $a_j$ prefer being matched with each other to the matched one in $A(t-1)$. In both cases, $A(t-1) \in M^*$ does not hold any more. Thus the claim is proved. 

Next we prove Lemma \ref{lem:generalize_preserve}. Suppose $\cap_{s=0}^{h_t} \cE_{t-s}$ and $\exists s' \in \set{0,1,\ldots,h_t}, A(t-s'-1) \in M^*$. Then with the above claim, we can recursively show $A(t-s) \in M^*$ for $s=s',s'-1,\ldots,0$, thus $A(t) \in M^*$. 
% \begin{align*}
% 	\cE_{t-s} \cap A(t-s-1) \in M^* &\Rightarrow A(t-s) \in M^*\,, \\
% 	\cE_{t-s+1} \cap A(t-s) \in M^* &\Rightarrow A(t-s+1) \in M^* \,, \ldots\\
% 	\cE_{t} \cap A(t-1) \in M^* &\Rightarrow A(t) \in M^* \,.
% \end{align*}
\end{proof}

With Lemma \ref{lem:generalize_preserve}, the stable regret can be bounded by
% \begin{align}
% {R}_i(T) \le& \mathbb{E}  \left[\sum_{t=1}^T \bOne{\mathop{\cap}\limits_{s=0}^{h_t} \bracket{\cE_{t-s} \cap \set{A(t-s-1) \notin M^*}}}   \right.  \notag \\
% &\left.  +  \sum_{t=1}^T\sum_{s=0}^{h_t} \bOne{ \urcorner \cE_{t-s} }      \right]\cdot \Delta_{i,\max} \label{eq:bound:TS} \,.
% \end{align}
\begin{align}
{R}_i(T) \le& \mathbb{E}  \left[\sum_{t=1}^T \bOne{\mathop{\cap}\limits_{s=0}^{h_t} \bracket{\cE_{t-s} \cap \set{A(t-s-1) \notin M^*}}}    +  \sum_{t=1}^T\sum_{s=0}^{h_t} \bOne{ \urcorner \cE_{t-s} }      \right]\cdot \Delta_{i,\max}  \,.
\end{align}
The first term can be regarded as the time required for the decentralized market to converge to stable matching with correct rankings. We can directly apply \cite[Lemma 14]{liu2020bandit} to bound this term. 
However, the analysis for the second term, which reflects the number of rounds needed to learn correct rankings,
 should depend on the special nature of TS and is quite different from that for UCB. 

 In the following, we bound a simple version with $h_t=0$ for the second term. 
For any $i\in[N]$ and $j\in[K]$, let $N_{i,j}(t) = \sum_{\tau<t}\bOne{ \bar{A}_i(\tau) = j}$ be the number of observations for $\mu_{i,j}$ before the start of time $t$. 
Note the event $\urcorner \cE_{t}$ implies that there exist $i\in[N]$ and $j,{j'}\in S_i(t)$ such that $j= \argmax_{k\in S_i(t)}\theta_{i,k}(t)$ and $\mu_{i,j}<\mu_{i,j'}$. Define $\cA_{i,j}(t) = \set{\abs{ \theta_{i,j}(t) - {\mu}_{i,j}} > \sqrt{6\log T/{N_{i,j}(t)}}}$ and $\cB_{i,j,j'}(t) = \set{\Delta_{i,j,j'} < \abs{ \theta_{i,j}(t) - \mu_{i,j} } + 2\varepsilon} $. Then, 
% \begin{small}
% \begin{align}
% &\sum_{t=1}^T \PP{\urcorner \cE_{t}} \le \sum_{i,j,j':\mu_{i,j}<\mu_{i,j'}}\sum_{t=1}^T \left(  \PP{\cA_{i,j}(t)} \right. \notag \\
%  &\left. +\PP{  \urcorner \cB_{i,j,j'}(t), j =\argmax_{k \in S_i(t)}\theta_{i,k}(t), \theta_{i,j}(t)>\theta_{i,j'}(t),j'\in S_i(t) }  \right.  \notag \\
% &\left.  + \PP{\urcorner \cA_{i,j}(t),  \cB_{i,j,j'}(t), j =\argmax_{k \in S_i(t)}\theta_{i,k}(t)  }  \right) \label{eq:bound:TS}
% \end{align}
% \end{small}
% \begin{small}

% \end{small}

% \begin{align}
% &\EE{\sum_{t=1}^T \bOne{\urcorner \cE_{t}}} \le \sum_{i,j,j':\mu_{i,j}<\mu_{i,j'}} \mathbb{E}\left[ \sum_{t=1}^T \left(  \bOne{\cA_{i,j}(t)} \right.\right. \notag \\
%  &\left.\left. +\bOne{  \urcorner \cB_{i,j,j'}(t), j =\argmax_{k \in S_i(t)}\theta_{i,k}(t), j'\in S_i(t) }  \right.\right.  \notag \\
% &\left.\left.  + \bOne{\urcorner \cA_{i,j}(t),  \cB_{i,j,j'}(t), j =\argmax_{k \in S_i(t)}\theta_{i,k}(t)  }  \right)\right] \label{eq:bound:TS}
% \end{align}
\begin{align}
\EE{\sum_{t=1}^T \bOne{\urcorner \cE_{t}}} \le \sum_{i,j,j':\mu_{i,j}<\mu_{i,j'}} \mathbb{E}\left[ \sum_{t=1}^T \left(  \bOne{\cA_{i,j}(t)} +\bOne{  \urcorner \cB_{i,j,j'}(t), j =\argmax_{k \in S_i(t)}\theta_{i,k}(t), j'\in S_i(t) }  \right.\right.  \notag \\
\left.\left.  + \bOne{\urcorner \cA_{i,j}(t),  \cB_{i,j,j'}(t), j =\argmax_{k \in S_i(t)}\theta_{i,k}(t)  }  \right)\right] \label{eq:bound:TS}
\end{align}

The first term in \eqref{eq:bound:TS} can be regarded as a bad event that the sampled parameters fall outside the confidence rigion and can be bounded by a constant. The last term represents the number of observations for $a_{j}$ required to distinguish `bad' arm $a_j$ from `good' arm $a_{j'}$ and can be bounded by $O(\log T/(\rho \Delta^2))$. The second term, a unique term for TS, comes from what we discussed in Section \ref{sec:discuss} and reflects the number of observations for $a_{j'}$ needed to guarantee $\abs{\theta_{i,j'}-\mu_{i,j'}} \le \varepsilon $. 
% For this term, once $\abs{\theta_{i,j'}-\mu_{i,j'}} \le \varepsilon $, we can get an observation for $a_{j'}$ with probability larger than $\rho$. 
By elaborately separating the horizon with slices and computing the probability of being successfully matched with $a_{j'}$, this term can be bounded by $O\bracket{1/\bracket{\rho \varepsilon^6}}$.

\section{Proof of Theorem \ref{thm:bound}}
\label{sec:proof:dcts}

Before the main proof, we first introduce some notations that will be used. 

Recall that $M^*$ is the set of all stable matchings between the player set $\cN$ and the arm set $\cK$, $N_{i,j}(t) = \sum_{\tau<t}\bOne{ \bar{A}_i(\tau) = j}$ is the number of observations of $\mu_{i,j}$ before the start of round $t$. 
For any player $p_i$ and arm $a_j$, we further define $\hat{\mu}_{i,j}(t) = \frac{1}{N_{i,j}(t)}\sum_{\tau<t: \bar{A}_i(\tau) = j } X_{i,j}(\tau)$ be the empirical mean outcome of $\mu_{i,j}$, $a_{i,j}(t)$ and $b_{i,j}(t)$ be the value of $a_{i,j}$ and $b_{i,j}$, respectively, before the start of round $t$.

The stable regret for player $p_i$ can be bounded as
\begin{align}
{R}_i(T) &\le \Delta_{i,\max}\cdot  \EE{\sum_{t=1}^T \bOne{A(t) \notin M^*}} \notag \\
% &\le \Delta_{i,\max} \EE{ \sum_{t=1}^T \bracket{\sum_{s=0}^{h_t} \bOne{\urcorner \cE_{t-s}  } + \PP{\cap_{s=0}^{h_t} \bracket{\cE_{t-s} \cap \set{A(t-s-1) \notin M^*}}} } }\label{eq:introduceht}\\
&\le \Delta_{i,\max} \set{ \EE{\sum_{t=1}^T \sum_{s=0}^{h_t}\bOne{\urcorner \cE_{t-s}  } } + \EE{\sum_{t=1}^T \bOne{\cap_{s=0}^{h_t} \bracket{\cE_{t-s} \cap \set{A(t-s-1) \notin M^*}}}}} \label{eq:introduceht} \\
&\le \Delta_{i,\max} \set{\sum_{s=0}^{h_T} \EE{\sum_{t=1}^T \bOne{\urcorner \cE_{t}}}  + \sum_{t=1}^T \PP{\cap_{s=0}^{h_t} \bracket{\cE_{t-s} \cap \set{A(t-s-1) \notin M^*}}}} \label{eq:decen:boundall}\\
&\le \Delta_{i,\max} \set{ (h_T+1)NK^2 \bracket{ 4+\frac{6\log T}{\rho(\Delta-2\varepsilon)^2} + \frac{C}{\rho }\cdot\frac{1}{\varepsilon^6} } + \sum_{t=1}^T \bracket{1-\rho^{N^4}}^{\lfloor h_t/N^4 \rfloor} }  \label{eq:decen:replace} \\
&\le \Delta_{i,\max} \set{ \frac{2N^5K^2\log T}{\rho^{N^4}} \bracket{ 4+\frac{6\log T}{\rho (\Delta-2\varepsilon)^2} + \frac{C}{\rho}\cdot\frac{1}{\varepsilon^6} } + 6 + \frac{6N^4}{\rho^{N^4}} } \label{eq:decen:B}\\
&=O\bracket{ \frac{N^5K^2 \log^2 T}{\rho^{N^4}\Delta^2}\cdot \Delta_{i,\max} } \,,\notag
\end{align}
where \eqref{eq:introduceht} comes from the result of Lemma \ref{lem:generalize_preserve},  \eqref{eq:decen:replace} comes from the result of Lemma \ref{lem:decen:TSpart} and Lemma \ref{lem:decen:matchingpart}, \eqref{eq:decen:B} holds by setting $h_t = \min\set{t, 2\lceil \frac{N^4 \log T}{\rho^{N^4}} \rceil }  -1$ and thus
\begin{align*}
\sum_{t=1}^T \bracket{1-\rho^{N^4}}^{\lfloor h_t/N^4 \rfloor} \le 3\sum_{t=1}^T \exp\bracket{ -h_t \rho^{N^4}/N^4 }\le 6T\exp\bracket{-2\lceil \frac{N^4 \log T}{\rho^{N^4}} \rceil \cdot \frac{\rho^{N^4}}{2N^4}} +\frac{6N^4}{\rho^{N^4}}\le 6+\frac{6N^4}{\rho^{N^4}}\,.
\end{align*}

% \begin{lemma}
% \label{lem: preserve}
% If $A(t-1) \in M^*$ and $\cE_t$ holds, then $A(t) = A(t-1) \in M^*$.
% \end{lemma}

% \begin{proof}
% Suppose $A(t)\neq A(t-1)$, then there must exist a player $p_i$ such that $A_i(t) = a_j \neq A_i(t-1)$. According to the Algorithm \ref{alg:CA-TS}, there must be $a_j \in S_i(t)$. 

% Since $A(t-1) \in M^*$, we have $A_{i}(t-1) \in S_i(t)$. 
% According to the event $\cE_t$, there must be $\mu_{i,j} > \mu_{i,A_{i}(t-1)}$, namely the player $p_i$ prefers $a_j$ to $A_{i}(t-1)$. Let $A'$ be the player matched with arm $a_j$ at time $t-1$. There are two cases for $A'$. 
% \begin{enumerate}
% 	\item[(a)] $A'$ is empty, namely arm $a_j$ is not matched at $t-1$. In this case, $(p_i, A_{i}(t-1), a_j)$ is a blocking pair at time $t-1$. Then we can conclude $A(t-1)\notin M^*$. 
% 	\item[(b)] $A'$ is non-empty. Since $a_j$ is in the plausible set $S_i(t)$, we can conclude $a_j$ prefers player $i$ to $A'$. Then $(p_i, A_{i}(t-1), a_j)$ is a blocking pair at time $t-1$, thus $A(t-1)\notin M^*$. 
% \end{enumerate}
% Above all, by contradiction, if $A(t-1) \in M^*$ and $\cE_t$ holds, then $A(t) = A(t-1) \in M^*$.
% \end{proof}

\begin{lemma}\label{lem:decen:TSpart}
(bound $\EE{\sum_{t=1}^T \bOne{\urcorner \cE_t}}$ in \eqref{eq:decen:boundall})
\begin{align*}
	\EE{\sum_{t=1}^T \bOne{\urcorner \cE_t}} &\le \sum_{i,j,j': \mu_{i,j}<\mu_{i,j'} } \bracket{ 4+\frac{6\log T}{\rho(\Delta_{i,j,j'}-2\varepsilon)^2} + \frac{C}{\rho}\cdot \frac{1}{\varepsilon^6} } \\
	&\le NK^2 \bracket{ 4+\frac{6\log T}{\rho(\Delta-2\varepsilon)^2} + \frac{C}{\rho}\cdot\frac{1}{\varepsilon^6} }\,,
\end{align*}
for any $\varepsilon$ such that $\Delta > 2\varepsilon$, where $\rho = \lambda^{N-1}(1-\lambda)$. 
\end{lemma}

\begin{proof}
	\begin{align*}
		\EE{\sum_{t=1}^T \bOne{\urcorner \cE_t}}  &= \EE{\sum_{t=1}^T\bOne{ \exists i \in [N]: \argmax_{k \in S_i(t)} \theta_{i,k}(t) \neq \argmax_{k \in S_i(t)}\mu_{i,k} }} \\
		&\le   \sum_{i,j,j': \mu_{i,j}<\mu_{i,j'} } \EE{\sum_{t=1}^T \bOne{ j=\argmax_{k \in S_i(t)}\theta_{i,k}(t), j'\in S_i(t) }}\,.
	\end{align*}
Recall event
\begin{align*}
\cA_{i,j}(t) = \set{\abs{ \theta_{i,j}(t) - {\mu}_{i,j}} > \sqrt{\frac{6\log T}{N_{i,j}(t)}}}\,, ~~~
\cB_{i,j,j'}(t) = \set{\Delta_{i,j,j'} < \abs{ \theta_{i,j}(t) - \mu_{i,j} } + 2\varepsilon} \,.
\end{align*}
Then we have 
\begin{align}
	\EE{\sum_{t=1}^T \bOne{ j=\argmax_{k \in S_i(t)}\theta_{i,k}(t) ,j'\in S_i(t) }}\notag 
	 \le & \EE{\sum_{t=1}^T \bOne{\cA_{i,j}(t)}}
	 + \EE{\sum_{t=1}^T \bOne{ \urcorner \cA_{i,j}(t),  \cB_{i,j,j'}(t), j =\argmax_{k \in S_i(t)}\theta_{i,k}(t)  } }\notag \\
	 &+ \EE{\sum_{t=1}^T \bOne{ \urcorner \cB_{i,j,j'}(t), j =\argmax_{k \in S_i(t)}\theta_{i,k}(t), j'\in S_i(t) }} \,.\label{eq:decen:index}
\end{align}
Next we will sequentially bound these terms.

\paragraph{The first term in \eqref{eq:decen:index}. }
\begin{align}
	\EE{\sum_{t=1}^T  \bOne{\cA_{i,j}(t)}} \le& \EE{\sum_{t=1}^T \bOne{\abs{ \theta_{i,j}(t) - \hat{\mu}_{i,j}(t)} > \sqrt{\frac{3\log T}{2N_{i,j}(t)}} \text{ or }\abs{ \hat{\mu}_{i,j}(t) - {\mu}_{i,j}} > \sqrt{\frac{3\log T}{2N_{i,j}(t)}}}} \notag \\
	% +\bOne{\abs{ \hat{\mu}_{i,j}(t) - {\mu}_{i,j}} > \sqrt{\frac{3\log T}{2N_{i,j}(t)}}} \notag\\
	% + \EE{\sum_{t=1}^T }  \notag\\
	\le&  \sum_{t=1}^T \sum_{w=1}^{T-1}\PP{N_{i,j}(t)=w, \abs{\theta_{i,j}(t) - \hat{\mu}_{i,j}(t)} > \sqrt{\frac{3\log T}{2N_{i,j}(t)}}} \notag \\
	&+\sum_{t=1}^T \sum_{w=1}^{T-1}\PP{N_{i,j}(t)=w, \abs{\hat{\mu}_{i,j}(t) - {\mu}_{i,j}} > \sqrt{\frac{3\log T}{2N_{i,j}(t)}}} 
	\notag\\
	=& \sum_{t=1}^T \sum_{w=1}^{T-1}\PP{N_{i,j}(t)=w} \cdot \left( \PP{\abs{\theta_{i,j}(t) - \hat{\mu}_{i,j}(t)} > \sqrt{\frac{3\log T}{2N_{i,j}(t)}} \mid N_{i,j}(t)=w} \right. \notag\\
	&\left. ~~~~~~~~~~~~~~~~~~~~~~~~~~~~~~~~~~~~~~~~+\PP{\abs{\hat{\mu}_{i,j}(t) - {\mu}_{i,j}} > \sqrt{\frac{3\log T}{2N_{i,j}(t)}} \mid N_{i,j}(t)=w} \right) \notag \\
	\le& \sum_{t=1}^T \sum_{w=1}^{T-1}\PP{N_{i,j}(t)=w} \cdot 4\exp\bracket{-3\log T}\label{eq:decen:beta_concen} \\
	\le& \sum_{t=1}^T \frac{4}{T} = 4 \notag\,,
\end{align}
where \eqref{eq:decen:beta_concen} comes from the result in Lemma \ref{lem:beta:concen} and Lemma \ref{lem:chernoff}.

% \paragraph{The second term in \eqref{eq:decen:index}. }
% \begin{align}
% 	% \sum_{t=1}^T \PP{\cB_{i,j}(t)} &= \sum_{t=1}^T \PP{\abs{ \hat{\mu}_{i,j}(t) - {\mu}_{i,j}} > \sqrt{\frac{3\log T}{2N_{i,j}(t)}}} \notag\\
% 	% &\le  \sum_{t=1}^T \sum_{w=1}^{T-1}\PP{N_{i,j}(t)=w, \abs{\hat{\mu}_{i,j}(t) - {\mu}_{i,j}} > \sqrt{\frac{3\log T}{2N_{i,j}(t)}}}\notag\\
% 	% &= \sum_{t=1}^T \sum_{w=1}^{T-1}\PP{N_{i,j}(t)=w} \cdot \PP{\abs{\hat{\mu}_{i,j}(t) - {\mu}_{i,j}} > \sqrt{\frac{3\log T}{2N_{i,j}(t)}} \mid N_{i,j}(t)=w} \notag\\
% 	% &\le \sum_{t=1}^T \sum_{w=1}^{T-1}\PP{N_{i,j}(t)=w} \cdot 2\exp\bracket{-3\log T}\label{eq:decen:hatmu_concen} \\
% 	% &\le \sum_{t=1}^T \frac{2}{T} = 2 \notag\,,
% \end{align}
% where \eqref{eq:decen:hatmu_concen} comes from the result in Lemma \ref{lem:chernoff}. 

\paragraph{The second term in \eqref{eq:decen:index}. } For this term, we claim that $\set{ \urcorner \cA_{i,j}(t), \cB_{i,j,j'}(t) } \Rightarrow  \set{N_{i,j}(t) \le \frac{6\log T}{(\Delta_{i,j,j'} - 2\varepsilon)^2} }$. 
Next we will prove this claim by contradiction. 

Suppose $N_{i,j}(t) > \frac{6\log T}{(\Delta_{i,j,j'} - 2\varepsilon)^2}$, then 
\begin{align*}
	\abs{\theta_{i,j}(t) - \mu_{i,j}}  \le \sqrt{\frac{6\log T}{N_{i,j}(t)}} <  \Delta_{i,j,j'} - 2\varepsilon \,,
\end{align*}
where the first inequality comes from event $\urcorner \cA_{i,j}(t)$. This result contradicts the event $\cB_{i,j,j'}(t)$, thus the claim has been proved. 

Above all, for this term, 
\begin{align*}
	\EE{\sum_{t=1}^T \bOne{ \urcorner \cA_{i,j}(t), \cB_{i,j,j'}(t), j =\argmax_{k \in S_i(t)}\theta_{i,k}(t)}}  &\le \EE{\sum_{t=1}^{T}\bOne{j= \argmax_{k\in S_i(t)} \theta_{i,k}(t), N_{i,j}(t)\le \frac{6\log T}{(\Delta_{i,j,j'} - 2\varepsilon)^2} }}\\
	&\le \frac{1}{\rho} \EE{\sum_{t=1}^{T}\bOne{j= \bar{A}_i(t), N_{i,j}(t)\le \frac{6\log T}{(\Delta_{i,j,j'} - 2\varepsilon)^2} }} \\
	&\le \frac{1}{\rho}\frac{6\log T}{(\Delta_{i,j,j'} - 2\varepsilon)^2} \,,
\end{align*}
where the second inequality is because when $j = \argmax_{k \in S_i(t)}\theta_{i,k}(t)$, if player $p_i$ selects arm $a_j$ and all of other players select the arm attempted in the last round, then $j= \bar{A}_i(t)$. This sufficient condition for $\set{j= \bar{A}_i(t)}$ holds with probability larger than $\rho = \lambda^{N-1}(1-\lambda)$. The last inequality holds since when $j= \bar{A}_i(t)$, $N_{i,j}(t+1) =N_{i,j}(t)+1$. 

\paragraph{The third term in \eqref{eq:decen:index}.} 
Define event $\cF_{i,j',1}(t)$ as 
\begin{align}
    \cF_{i,j',1}(t) = &\left\{\forall \theta'_i \text{ satisfying } \abs{\theta'_{i,j'}-\mu_{i,j'}}\le \varepsilon  \text{ and } \theta'_{i,j}=\theta_{i,j}(t) \text{ for any } j \neq j',\text{ it holds that for any $k \in S_i(t),k \neq j'$ }, \theta'_{i,j'} > \theta'_{i,k}. \right. \notag \\
    &\left. \text{ Further, if we replace $\theta(t)$ with }\theta' = (\theta'_i, (\theta_{i'}(t))_{i'\neq i}), \text{ then }\bar{A}_i(t)=j' \text{ with probability larger than }\rho \right\}\,,\label{eq:F1}
\end{align}
and event $\cF_{i,j',2}(t)$ as
\begin{align}
	\cF_{i,j',2}(t) = \set{ \abs{\theta_{i,j'}(t) - \mu_{i,j'}} > \varepsilon }\,. \label{eq:F2}
\end{align}

We claim that if event $\set{\urcorner \cB_{i,j,j'}(t), j =\argmax_{k \in S_i(t)}\theta_{i,k}(t), j'\in S_i(t)}$ holds, then $\cF_{i,j',1}(t)$ and $\cF_{i,j',2}(t)$ hold. 
 
We first consider $\cF_{i,j',1}(t)$. According to $\urcorner \cB_{i,j,j'}(t)$, for any $\theta'_i$ satisfying the requirements in $\cF_{i,j',1}(t)$, there is
\begin{align*}
	\theta'_{i,j} = \theta_{i,j}(t) &\le \mu_{i,j} + \Delta_{i,j,j'} - 2\varepsilon = \mu_{i,j'}  - 2\varepsilon \le \theta'_{i,j'} + \varepsilon-2\varepsilon  < \theta'_{i,j'} \,.
\end{align*}
And for any other arm $k\in S_i(t), k \neq j,j'$, we have $\theta'_{i,k} =\theta_{i,k}(t)  \le \theta_{i,j}(t) = \theta'_{i,j}<\theta'_{i,j'}$. Further, if we replace $\theta(t)$ with $\theta' = (\theta'_i, (\theta_{i'}(t))_{i'\neq i})$, then with probability $1-\lambda$, player $p_i$ attempts to pull arm $j' = \argmax_{k}\theta'_{i,k}$. If all of other players selects the arm attempted in the last round, which happens with probability $\lambda^{N-1}$, then player $p_i$ will succeed. Above all, we have $\bar{A}_i(t)=j'$ with probability at least $\rho = (1-\lambda)\lambda^{N-1}$. 
Thus we conclude 
\begin{align*}
\set{\urcorner \cB_{i,j,j'}(t), j =\argmax_{k \in S_i(t)}\theta_{i,k}(t), j' \in S_i(t) }\Rightarrow \cF_{i,j',1}(t)\,.
\end{align*}

We then consider $\cF_{i,j',2}(t)$. By contradiction, suppose $\urcorner \cF_{i,j',2}(t) = \set{ \abs{\theta_{i,j'}(t) - \mu_{i,j'}} \le \varepsilon } $ holds, then $\theta_i(t)$ satisfies the requirment for $\theta'_i$ in $\cF_{i,j',1}(t)$. According to $\cF_{i,j',1}(t)$, we would have $\theta_{i,j'}(t) > \theta_{i,j}(t)$, which contradicts $j =\argmax_{k \in S_i(t)}\theta_{i,k}(t)$. Thus $\set{\urcorner \cB_{i,j,j'}(t), j =\argmax_{k \in S_i(t)}\theta_{i,k}(t) ,j' \in S_i(t) }\Rightarrow \cF_{i,j',2}(t)$.

Above all, the claim has been proved.

We further introduce some notations used in the following analysis. 
Let $\tau_{\ell}$ be the round at which $\cF_{i,j',1}(t)\land \neg \cF_{i,j',2}(t)$ occurs for the $\ell$-th time. Denote $\rho(t)$ as the exact probability that $\bar{A}_{i}(t)=j'$. 
We construct a new counter $N'_{i,j'}$. At round $t$, if $\cF_{i,j',1}(t)\land \neg \cF_{i,j',2}(t)$ occurs and $\bar{A}_{i}(t)=j'$, we update $N'_{i,j'}(t+1)=N'_{i,j}(t)+1$ with probability $\rho/\rho(t)$. Otherwise $N'_{i,j'}(t+1)=N'_{i,j'}(t)$. Since when $\cF_{i,j',1}(t)\land \neg \cF_{i,j',2}(t)$ occurs, $\bar{A}_{i}(t)=j'$ holds with probability larger than $\rho$, we have $\rho/\rho(t)<1$. 
And $\PP{N'_{i,j}(t+1)\neq N'_{i,j}(t)\mid \cF_{i,j',1}(t)\land \neg \cF_{i,j',2}(t)} = \rho(t)\cdot \frac{\rho}{\rho(t)} = \rho$. Let $\eta_q$ be the round $t$ at which $N'_{i,j'}(t+1)=q = N'_{i,j'}(t)+1$
and $\eta_0 = \tau_0 = 0$. 

Above all, for the third term in \eqref{eq:decen:index}, 
\begin{align}
	&\EE{\sum_{t=1}^T \bOne{ \urcorner \cB_{i,j,j'}(t), j =\argmax_{k \in S_i(t)}\theta_{i,k}(t), j' \in S_i(t) }} \notag \\
	\le &\EE{\sum_{t=1}^T \bOne{\cF_{i,j',1}(t) \wedge \cF_{i,j',2}(t)}} \notag \\
	\le& \sum_{q\ge 0} \EE{\sum_{t=\eta_{q}+1}^{\eta_{q+1}}\bOne{\cF_{i,j',1}(t) \wedge \cF_{i,j',2}(t)}}  \notag \\
	=& \sum_{q\ge 0}\left\{ \sum_{\ell_1 \ge 0}^{\infty} \PP{\eta_{q} = \tau_{\ell_1}} \times \sum_{\ell_2 \ge 0}^{\infty} \PP{\eta_{q+1} = \tau_{\ell_1+\ell_2+1} \mid \eta_{q} = \tau_{\ell_1}} \times \sum_{\ell=\ell_1 }^{\ell_1 +\ell_2} \EE{\sum_{t=\tau_{\ell}+1}^{\tau_{\ell+1}}\bOne{\cF_{i,j',1}(t) \wedge \cF_{i,j',2}(t) \mid  \eta_{q}\le \tau_{\ell}<\eta_{q+1} } }\right\} \notag \\
	 % here
	\le & \sum_{q\ge 0}\left\{ \sum_{\ell_1 \ge 0}^{\infty} \PP{\eta_{q} = \tau_{\ell_1}} \times \sum_{\ell_2 \ge 0}^{\infty} \rho(1-\rho)^{\ell_2} 
	 \times \sum_{\ell=\ell_1 }^{\ell_1 +\ell_2}  \set{\EE{\sup_{t\ge \eta_{q}+1} \frac{1}{\PP{\abs{\theta_{i,j'}(t)-\mu_{i,j'}}\le \varepsilon \mid \cH_{t}}}  }-1} 
	 \right\} \label{eq:since:lem:TSexponential} \\
	% \le &  \sum_{q\ge 0} \set{\EE{\sup_{t\ge \eta_{i,j',q}+1} \frac{1}{\PP{\abs{\theta_{i,j'}(t)-\mu_{i,j'}}\le \varepsilon \mid \cH_{t}}}  }-1}  \\
	\le & \sum_{q\ge 0}\left\{ \sum_{\ell_1 \ge 0}^{\infty} \PP{\eta_{q} = \tau_{\ell_1}} \times \sum_{\ell_2 \ge 0}^{\infty}\ell_2 \rho(1-\rho)^{\ell_2}\times 
	   \set{\EE{\sup_{t\ge \eta_{q}+1} \frac{1}{\PP{\abs{\theta_{i,j'}(t)-\mu_{i,j'}}\le \varepsilon \mid \cH_{t}}}  }-1} 
	 \right\} \notag  \\
	\le& \frac{1}{\rho}\cdot \sum_{q\ge 0} \set{\EE{\sup_{t\ge \eta_{q}+1} \frac{1}{\PP{\abs{\theta_{i,j'}(t)-\mu_{i,j'}}\le \varepsilon \mid \cH_{t}}}  }-1} \label{eq:since:rho} \\
	\le& \frac{1}{\rho} \bracket{\sum_{q=0}^{\lceil8/\varepsilon^2\rceil-1} \bracket{c\varepsilon^{-4}} + \sum_{q \ge \lceil 8/\varepsilon^2\rceil} e^{-\varepsilon^2 q /8}\bracket{ c'\varepsilon^{-4} } } \le \frac{C}{\rho \varepsilon^6}\,, \label{eq:since:techsum}
\end{align}
where 
\eqref{eq:since:lem:TSexponential} comes from the result in Lemma \ref{lem:TS:exponential} and $\cH_{t}$ is the history of observations at time $t$, \eqref{eq:since:rho} holds by computing the expectation for Geometrically distributed random variable and $\sum_{\ell_1 \ge 0}\PP{\eta_q = \tau_{\ell_1}} = 1$, \eqref{eq:since:techsum} comes from the result in Lemma \ref{fact:nips20}. Here $C$ is a universal constant.

Combining all these three terms in \eqref{eq:decen:index}, we have proved the result in Lemma \ref{lem:decen:TSpart}, 
\begin{align*}
\EE{\sum_{t=1}^T \bOne{\urcorner \cE_t}} &\le \sum_{i,j,j': \mu_{i,j}<\mu_{i,j'} } \bracket{ 4+\frac{6\log T}{\rho(\Delta_{i,j,j'}-2\varepsilon)^2} + \frac{C}{\rho}\cdot \frac{1}{\varepsilon^6} } 
	\le NK^2 \bracket{ 4+\frac{6\log T}{\rho(\Delta-2\varepsilon)^2} + \frac{C}{\rho}\cdot\frac{1}{\varepsilon^6} } \,.
\end{align*}
\end{proof}

\begin{lemma}(bound $\PP{\cap_{s=0}^{h_t} \bracket{\cE_{t-s} \cap \set{A(t-s-1) \notin M^*}}}$ in \eqref{eq:decen:boundall}) \label{lem:decen:matchingpart}
According to \cite[Lemma 14]{liu2020bandit}, we have 
\begin{align*}
\PP{\cap_{s=0}^{h_t} \bracket{\cE_{t-s} \cap \set{A(t-s-1) \notin M^*}}} \le \bracket{1-\rho^{N^4}}^{\lfloor h_t/N^4 \rfloor} \,.
\end{align*}
\end{lemma}

\begin{lemma}\label{lem:TS:exponential}
    Fix player $p_i$ and arm $a_{j}$. Events $\cF_{i,j,1}(t)$ and $\cF_{i,j,2}(t)$ are defined as in \eqref{eq:F1},\eqref{eq:F2}. 
    Let $\tau_{\ell}$ be the round $t$ at which $\cF_{i,j,1}(t)\land \neg \cF_{i,j,2}(t)$ occurs for the $\ell$-th time, let $\tau_{0}=0$. Then for Algorithm \ref{alg:CA-TS}, we have
    \begin{align*}
        \EE{\sum_{t=\tau_{\ell}+1}^{\tau_{\ell+1}}\bOne{\cF_{i,j,1}(t)\wedge  \cF_{i,j,2}(t)}} \le \EE{\sup_{t\ge \tau_{\ell}+1} \frac{1}{\PP{\abs{\theta_{i,j}(t)-\mu_{i,j}}\le \varepsilon \mid \cH_{t}}}  } - 1 \,,
    \end{align*} 
    where $\cH_{t}$ is the history of observations at time $t$. 
\end{lemma}
\begin{proof}
    Note that the event $\cF_{i,j,1}(t)$ and $\cF_{i,j,2}(t)$ are independent conditioned on history $\cH_{t}$. 
    Let $\tau_{\ell,\ell'}$ be the round $t$ at which the event $\cF_{i,j,1}(t)$ happens for the $\ell'$-th time after round $\tau_{\ell}+1$. 
    Then, 
    \begin{align*}
        \EE{\sum_{t=\tau_{\ell}+1}^{\tau_{\ell+1}}\bOne{\cF_{i,j,1}(t)\wedge \cF_{i,j,2}(t)}}  = \EE{\sum_{\ell'\ge 1}(\ell'-1)\PP{\neg \cF_{i,j,2}(\tau_{\ell,\ell'}) \mid \cH_{\tau_{\ell,\ell'}}} \prod_{\ell''=1}^{\ell'-1}\PP{ \cF_{i,j,2}(\tau_{\ell,\ell''}) \mid \cH_{\tau_{\ell,\ell''}} }}\,,
    \end{align*}
    
We observe that the RHS of the above equality is the expectation of a time-varying geometric distribution with the success probability of the $\ell''$-th trial being $\PP{\neg \cF_{i,j,2}(\tau_{\ell,\ell''}) \mid \cH_{\tau_{\ell,\ell''}}  }$. 
According to the definition in \eqref{eq:F2}, this success probability satisfies 
    \begin{align*}
        \inf_{t \ge \tau_{\ell}+1} \PP{\neg \cF_{i,j,2}(t) \mid \cH_{t} } = \inf_{t \ge \tau_{\ell}+1}  \PP{\abs{ \theta_{i,j}(t)-\mu_{i,j} }\le \varepsilon \mid \cH_{t} }\,.
    \end{align*}
    Then according to the monotonicity of the expectation, it can be upper bounded by
    \begin{align*}
        &\EE{\sum_{\ell'\ge 1}(\ell'-1)\PP{\neg \cF_{i,j,2}(\tau_{\ell,\ell'}) \mid \cH_{\tau_{\ell,\ell'}}} \prod_{\ell''=1}^{\ell'-1}\PP{ \cF_{i,j,2}(\tau_{\ell,\ell''}) \mid \cH_{\tau_{\ell,\ell''}} }}\\
         \le &\EE{\sup_{t \ge \tau_{\ell}+1} \frac{1}{ \PP{\abs{ \theta_{i,j}(t)-\mu_{i,j} }\le \varepsilon \mid \cH_{t} }} }   -1 \,.
    \end{align*}

\end{proof}

\section{CA-TS with Gaussian Priors}
\label{sec:CA-TS:gaussian}

For completeness, we present the algorithm of CA-TS with Gaussian priors in Algorithm \ref{alg:CA-TS-gaussian}. The stable regret upper bound for this algorithm is shown in the following Theorem \ref{thm:bound:gau}, which has the same order as Theorem \ref{thm:bound} with only difference on the constant coefficients.

\begin{algorithm}[thb!]
    \caption{CA-TS with Gaussian priors}\label{alg:CA-TS-gaussian}
    \begin{algorithmic}[1]
    \STATE Input: player set $\cN$, arm set $\cK$, parameter $\lambda \in (0,1),$ \label{alggau:cats:input}
    \STATE Initialize: let each player $p_i$ play each arm $a_j$ once and set $\hat{\mu}_{i,j}=X_{i,j}, N_{i,j}=1$ \label{alggau:cats:initial}
    \FOR{$t=1,2,\cdots$}
        \FOR{$p_i \in \cN$}
        \STATE $\forall a_j:$ sample $\theta_{i,j}(t)\sim \cN(\hat{\mu}_{i,j},1/N_{i,j})$ \label{alggau:cats:sample}
            \STATE Independently draw $D_{i}(t)\sim\mathrm{Bernoulli}(\lambda)$ \label{alggau:cats:draw}
            \IF{$D_{i}(t)==0$} \label{alggau:cats:select:start}
                \STATE Construct plausible set $S_{i}(t)$ for player $p_i$ \label{alggau:cats:plausible}
                \begin{align*}
                    S_{i}(t) = \set{j: \pi_{j,i}\ge \pi_{j,i'} \text{ where } \bar{A}_{i'}(t-1)=j  }
                \end{align*}
                \STATE Pull $A_{i}(t) \in \argmax_{j \in S_{i}(t)}\theta_{i,j}(t)$ \label{alggau:cats:select:end}
            \ELSE  
                \STATE Pull $A_{i}(t) = A_{i}(t-1)$ \label{alggau:cats:delay}
            \ENDIF
            \IF{$p_i$ wins conflict} \label{alggau:cats:update:start}
                \STATE $\bar{A}_{i}(t) = A_{i}(t)$ \label{alggau:cats:mark}
                \STATE $\hat{\mu}_{i,A_i(t)} = \frac{\hat{\mu}_{i,A_{i}(t)}\cdot N_{i,A_{i}(t)} +  X_{i,A_{i}(t)}(t)}{N_{i,{A}_{i}(t)}+1} $ 
                \STATE $N_{i,A_{i}(t)} = N_{i,A_{i}(t)}+1$
            \ENDIF \label{alggau:cats:update:end}
        \ENDFOR
    \ENDFOR
    \end{algorithmic}
\end{algorithm}

\begin{theorem}\label{thm:bound:gau}
Let $\rho = \lambda^{N-1}(1-\lambda)$ and the rewards in each round are $1$-subgaussian random variables. Following Algorithm \ref{alg:CA-TS-gaussian}, the cumulative stable regret of each player $p_i$ satisfies

\begin{align}
R_i(T) 
% \le& \sum_{t=1}^T \PP{A(t) \notin M^*} \Delta_{i,\max} \notag\\
\le &  \left\{ \frac{2N^5K^2\log T}{\rho^{N^4}} \bracket{ 4+\frac{8\log T}{\rho (\Delta-2\varepsilon)^2} + \frac{C}{\rho\varepsilon^6} }   + 6 + \frac{6N^4}{\rho^{N^4}} \right\}\cdot \Delta_{i,\max} \label{bound:ful:gau}\\
=&O\bracket{ \frac{N^5K^2 \log^2 T}{\rho^{N^4}\Delta^2}\cdot \Delta_{i,\max} } \label{bound:order:gau}\,,
\end{align}
for any $\varepsilon$ such that $\Delta - 2\varepsilon >0$, where $C$ is a universal constant and the maximum stable regret $\Delta_{i,\max}$ is re-defined as $\max_{j \in [K]}\max\set{\mu_{i,m_i}-\mu_{i,j}, \mu_{i,m_i}}$.
\end{theorem}

\begin{proof}
The proof of Theorem \ref{thm:bound:gau} is very similar to that of Theorem \ref{thm:bound}. Here we only state the difference with Appendix \ref{sec:proof:dcts}. 

When proving Lemma \ref{lem:decen:TSpart}, we slightly change the constant in the definition of event $\cA_{i,j}(t)$ as 
\begin{align*}
    \cA_{i,j}(t) = \abs{ \theta_{i,j}(t)- \mu{i,j}} > \sqrt{\frac{8\log T}{N_{i,j}(t)}} \,.
\end{align*}
Based on this new definition, we use the result of Lemma \ref{fact:gaussian:concen} and \ref{fact:subgaussian:concen} to get the upper bound for the first term in \eqref{eq:decen:index} instead of Lemma \ref{lem:beta:concen} and \ref{lem:chernoff}. In addition, to get the result of Eq. \eqref{eq:since:techsum}, we use the upper bound of \cite[Eq. (4)]{perrault2020statistical} instead of the result of Lemma \ref{fact:nips20}.

\end{proof}

\section{Additional Experiments}\label{sec:add:exp}
\subsection{Introduction to Baseline Algorithms}\label{sec:baselines}

\paragraph{CA-UCB}  The CA-UCB algorithm proposed by \citeauthor{liu2020bandit} \shortcite{liu2020bandit} is a UCB-type algorithm and maintains a UCB index for each preference value to select arms. We follow the parameter selection of the original paper and set $\lambda=0.1$ for all experiments, same for our CA-TS.

\paragraph{PhasedETC (P-ETC)} The PhasedETC (P-ETC) algorithm proposed by \citeauthor{basu21beyond} \shortcite{basu21beyond} is an ETC-type algorithm, which switches between phases of exploration and exploitation. We set $\epsilon=0.2$ as in the original paper.

\paragraph{Decentralized ETC (D-ETC)} The Decentralized ETC (D-ETC) algorithm proposed by \citeauthor{liu2020competing} \shortcite{liu2020competing} is another ETC-type algorithm requiring hyperparameter $H$ as exploration budget. 
% The selection of $H$ is not mentioned in their paper and we will report our choice later. 
 % in Section \ref{sec:exp:global}.
We test different values of $H$ in $\set{50,100,200,300,400}$ in market with global preferences (Section \ref{sec:exp:global}) and find that $H=200$ has the best performance in the cumulative market unstability, thus we fix this value in all experiments.

\paragraph{UCB-D4} The UCB-D4 algorithm proposed by \citeauthor{basu21beyond} \shortcite{basu21beyond} is only designed for markets satisfying the uniqueness consistency, where the stable matching is unique and robust. 
Here the robustness means removing any subset of matched pairs will not make the remaining matching unstable. 
So we only test the performances of UCB-D4 in the market with global preferences (Section \ref{sec:exp:global}), where the assumption of uniqueness consistency can be satisfied. 
We follow the experimental setting in the original paper and set $\beta = 1/(2K), \gamma=2$.

\subsection{Varying Heterogeneity of Players' Preferences}\label{sec:exp:beta}

One may concern that the performances of algorithms could be worse when the preferences of all players tend to be the same, since players may always attempt same arms and thus lead more conflicts. In this section, we follow the experimental setting of \citeauthor{liu2020bandit} \shortcite{liu2020bandit} to test the performances of algorithms in markets with different heterogeneity of players' preferences.

The market size is fixed with $N=5$ players and $K=5$ arms.
The preference rankings for arms are randomly generated while the generation of players' preference follow a random utility model \cite{ashlagi2017communication,liu2020bandit}. 
Specifically, we first sample an independent random variable $x_j \sim \mathrm{Uniform}([0,1])$ for each arm $a_j$ and an independent random variable $\epsilon_{i,j}\sim \mathrm{Logistic}(0,1)$ for each pair of player $p_i$ and arm $a_j$. Let $\bar{\mu}_{i,j} = \beta x_j + \epsilon_{i,j}$, where $\beta$ is a parameter to be determined later. Then the preference value $\mu_{i,j}$ is set as $ \frac{1}{K}\sum_{j' \in [K]} \bOne{\bar{\mu}_{i,j'}\le \bar{\mu}_{i,j}}$. The parameter $\beta$ decides the correlation of players' preferences. The larger $beta$, the more similar the players' preferences are. 

We test four choices of $\beta \in \set{0,10,50,100}$. When $\beta=100$, all players share the same preference over arms. 
The maximum cumulative stable regret among players and the cumulative market unstability of different algorithms under different values of $\beta$ are shown in Figure \ref{fig:beta}.

\begin{figure}[th!] 
\includegraphics[width=1\linewidth]{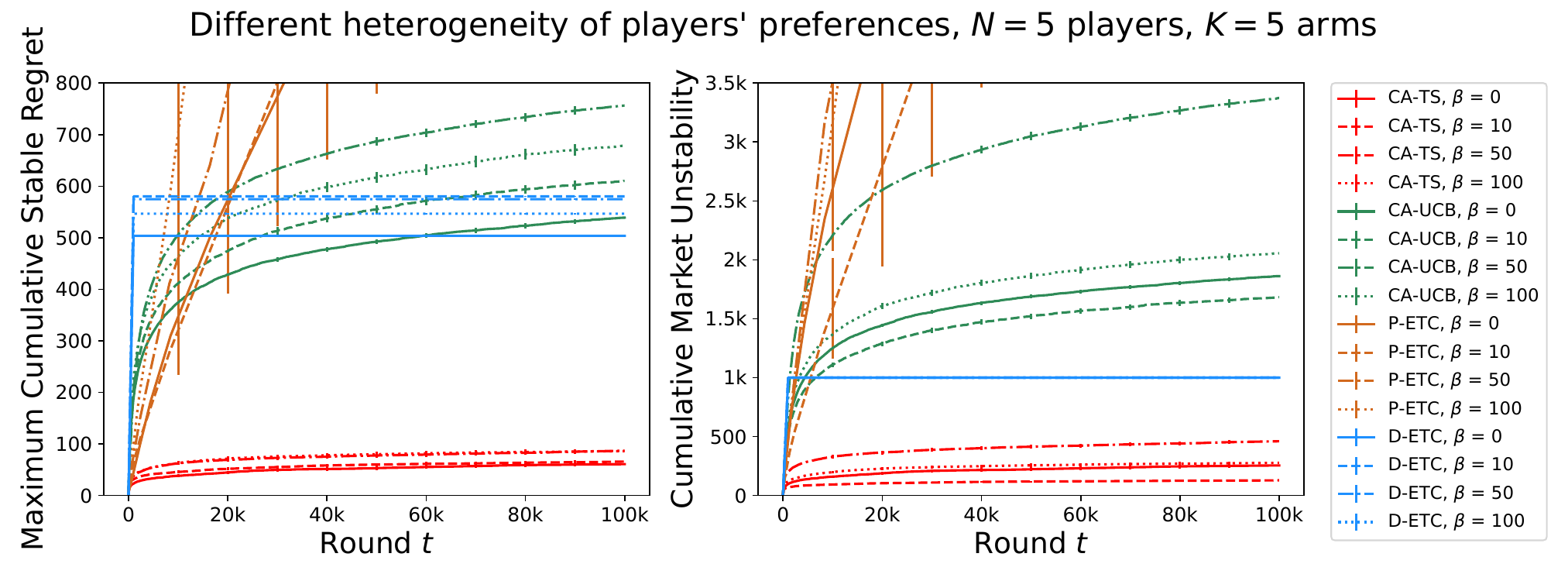}
  \caption{Experimental comparisons of our CA-TS with CA-UCB, P-ETC, and D-ETC in terms of maximum cumulative stable regret among players (left) and cumulative market unstability (right). 
  Markets of $N=5$ players and $K=5$ arms with different heterogeneity of players' preferences are tested.}
  \label{fig:beta}
\end{figure}

The CA-TS algorithm shows much better stable regret than all of other baselines under each value of $\beta$. Its performance only slightly changes when $\beta$ is changed. 
Compared with CA-TS, D-ETC pays higher regret and CA-UCB explores more rounds to find the true stable matching, the stable regret for which also varies more with change of $\beta$. 
The P-ETC algorithm still does not converge in all markets. 
Similar observation on these algorithms can also be found in terms of market unstability. 
It shows that there is no linear relationship between algorithms' market unstability and $\beta$. 
One may doubt why D-ETC shows the same market unstability under different $\beta$. It is because this algorithm forces players to explore arms for a fixed number of rounds without considering players' preferences.

% D-ETC achieves the same market unstability under different $\beta$, which is due to it forces players to explore arms for a fixed number of times without considering players' preferences.

% The D-ETC shows the same market unstability under different choices of $\beta$, which is due to it forces players to explore arms for a fixed number of times without considering players' preferences. 

% \newpage

\section{The Example Mentioned in Section \ref{sec:conclusion}}
\label{sec:conterexample}

In this section, we provide the example mentioned in Section \ref{sec:conclusion} where players suffer stable regret of order $O(T)$ when replacing UCB indices in \citeauthor{liu2020competing} \shortcite{liu2020competing} with sampled parameters from posterior distributions as TS. To be self-contained, we present the full algorithm in Algorithm \ref{alg:C-TS}, which is referred to as Centralized-TS. 
We assume the Gale-Shapley (GS) algorithm runs in a way where the arm side is the proposing side. This does not contradict the Centralized-UCB algorithm in \citeauthor{liu2020competing} \shortcite{liu2020competing} since their upper bound is on the player-pessimal stable regret, where the player-pessimal stable matching can be returned by this type of GS algorithm.

\begin{algorithm}[thb!]
    \caption{Centralized-TS}\label{alg:C-TS}
    \begin{algorithmic}[1]
    \STATE Input: player set $\cN$, arm set $\cK$, arms' preferences $(\pi_j)_{j \in [K]}$
    \STATE Initialize: $\forall i \in[N], j \in [K]: a_{i,j}=b_{i,j}=1$ \label{alg:cts:initial}
    \FOR{$t=1,2,\cdots$}
    	\STATE For any player $p_i$, arm $a_j$: $\theta_{i,j}(t)\sim \Beta(a_{i,j},b_{i,j})$
    	\STATE For any player $p_i$, compute estimated preference ranking $\hat{r}_i(t)$ according to $(\theta_{i,j}(t))_{j \in [K]}$
    	\STATE Compute matching $\bar{A}(t) = $Gale-Shapley$((\hat{r}_i(t))_{i \in [N]},(\pi_j)_{j \in [K]})$
        \FOR{$(p_i,a_j) \in \bar{A}(t)$}
            \STATE Observe reward $X_{i,j}(t)$
            \STATE $Y_i(t) \sim \mathrm{Bernoulli}(X_{i,j}(t))$
%             \STATE $ Y_i(t)=\left\{
% \begin{aligned}
% 1,&  ~\text{with probability } X_{i,j}(t)\\
% 0,&  ~\text{with probability } 1-X_{i,j}(t)
% \end{aligned}
% \right.
% $ 
                \STATE Update $a_{i,j} = a_{i,j} +Y_{i}(t)$, $b_{i,j} = b_{i,j}+(1-Y_{i}(t))$
        \ENDFOR
    \ENDFOR
    \end{algorithmic}
\end{algorithm}

We next provide the market example. There are $3$ players and $3$ arms. The preferences of market participants are shown as follows. 
For simplicity, we use $a_j > a_{j'}$ for player $p_i$ to represent player $p_i$ prefers arm $a_j$ to $a_{j'}$. Similarly, for arm $a_j$, we use $p_i > p_{i'}$ to represent $a_j$ prefers player $p_i$ to $p_{i'}$.

\begin{tasks}(2)
 \task player $p_1$: $a_3 >a_2 >a_1$
 \task arm $a_1$: $p_1>p_2>p_3$
 \task player $p_2$: $a_1 >a_3 >a_2$
 \task arm $a_2$: $p_1>p_2>p_3$
 \task player $p_3$: $a_2 >a_3 >a_1$
 \task arm $a_3$: $p_2>p_1>p_3$
\end{tasks}

The \textit{true} player-pessimal stable matching in this instance is $\set{(p_1,a_3),(p_2,a_1),(p_3,a_2)}$. 
In this example, once player $p_1$ wrongly estimates the ranking for arm $a_1$ and $a_2$ and all other rankings are correct, then the matching returned by GS is unstable and $p_1$ is matched with $a_1$ but never observes $\mu_{1,2}$. Thus with constant probability there would be $\theta_{1,2}(t)<\theta_{1,1}(t)$ and the returned matching is unstable. The stable regret for $p_1$ is thus of order $O(T)$. 
We also conduct experiments to verify this result in Figure \ref{fig:counter}.

\begin{figure}[th!] 
\includegraphics[width=0.24\linewidth]{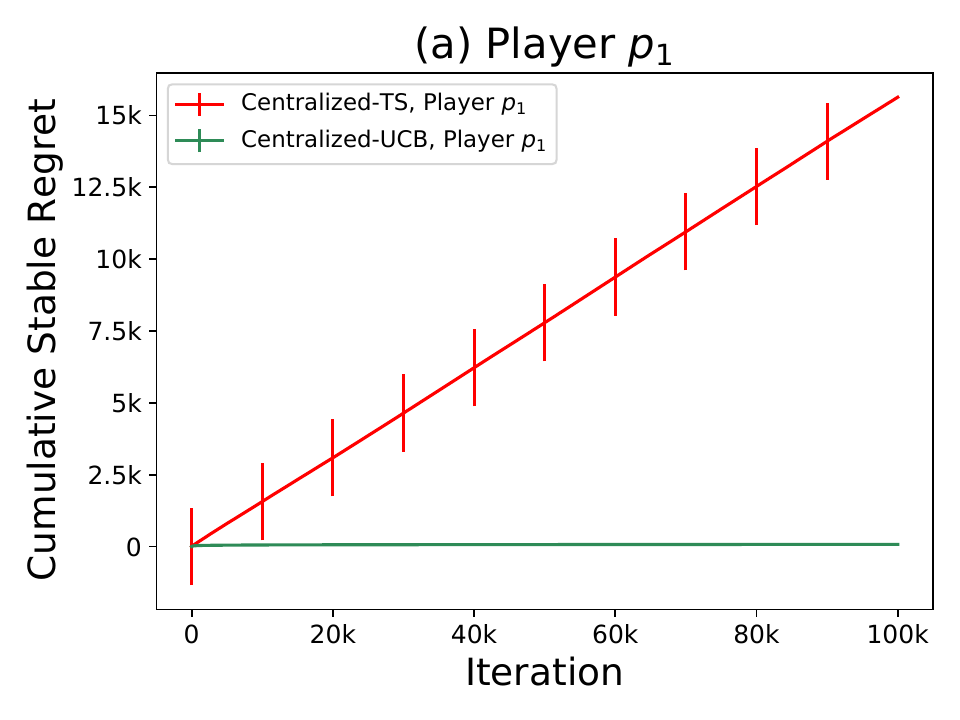}
\includegraphics[width=0.24\linewidth]{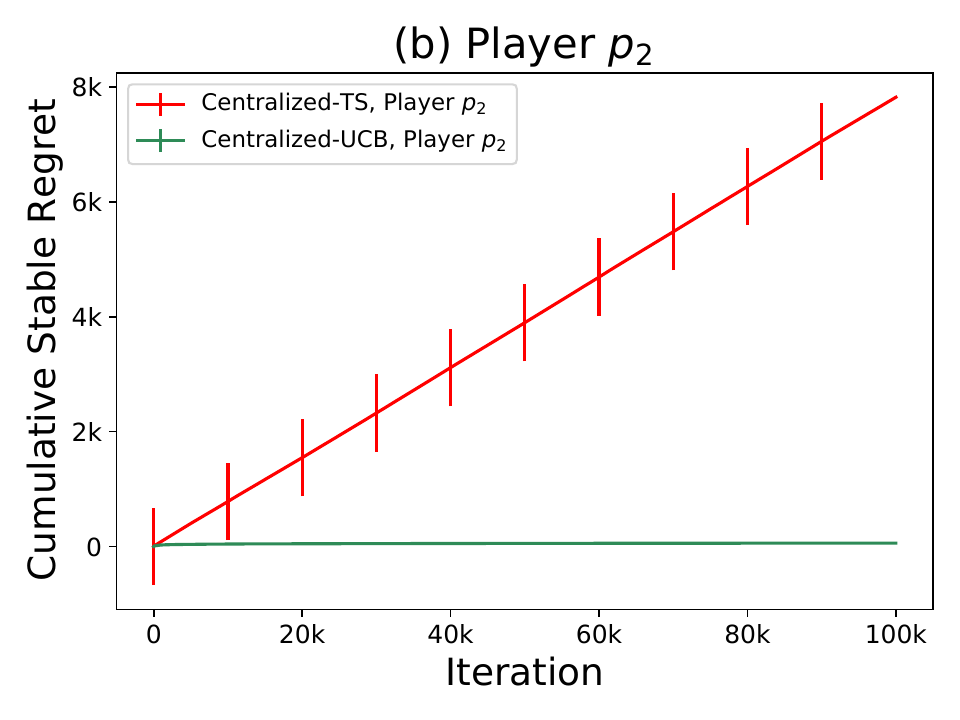}
\includegraphics[width=0.24\linewidth]{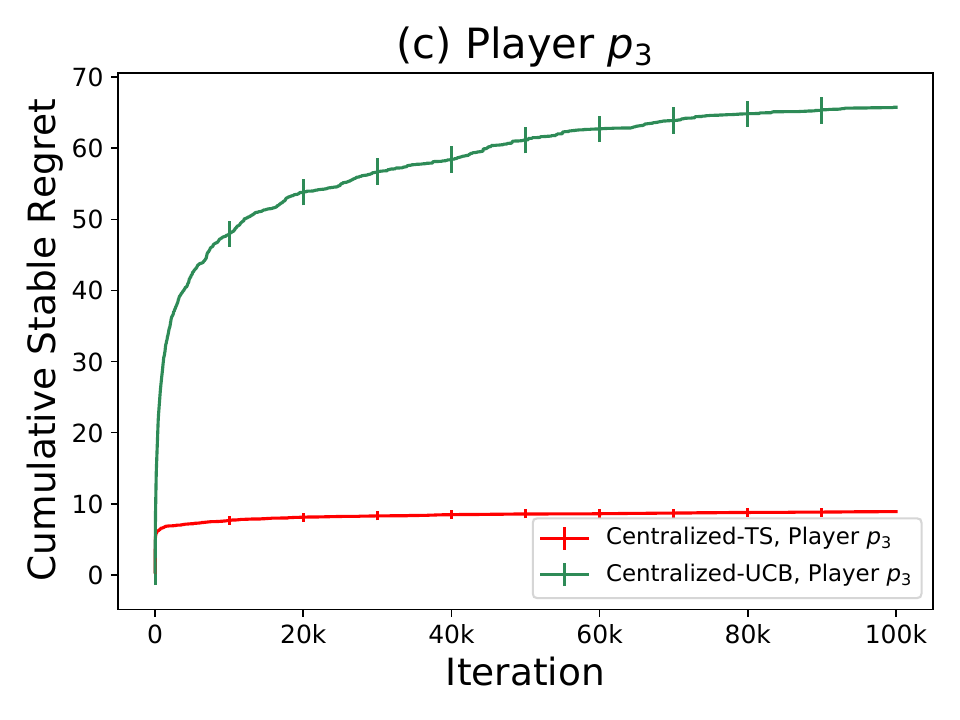}
\includegraphics[width=0.24\linewidth]{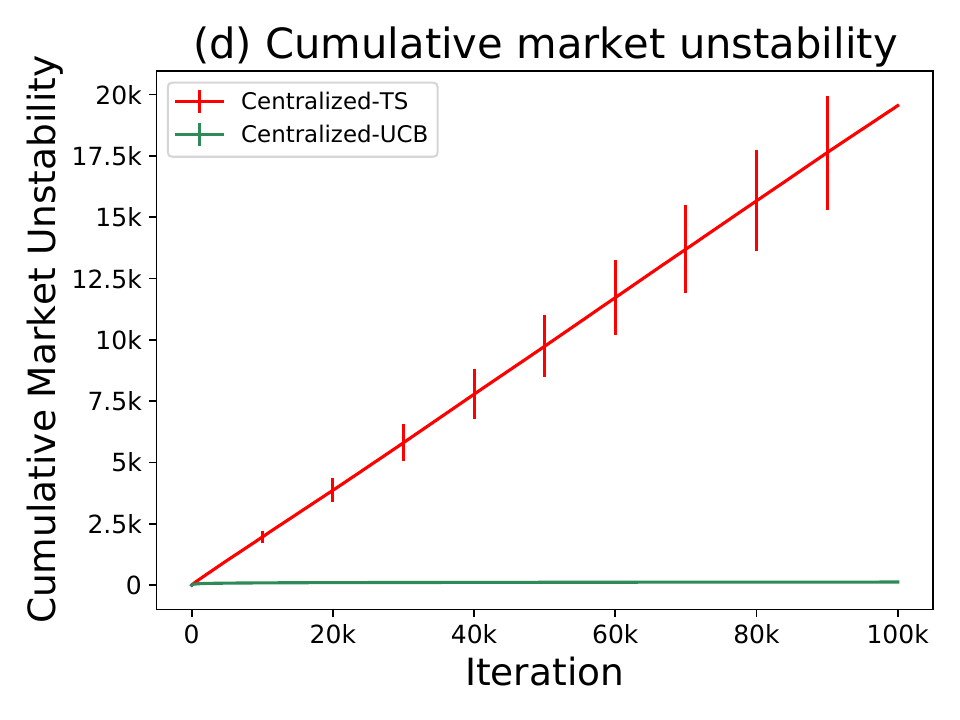}
  \caption{Experimental comparisons of Centralized-TS with Centralized-UCB in the above market example of $N=3$ players and $K=3$ arms.}
  \label{fig:counter}
\end{figure}

The results show that the Centralized-TS could not converge to a stable matching, which also suggest the TS-type algorithm is more restrictive than the UCB-type algorithm.
This example further illustrates the challenges to design a TS-type algorithm for the centralized matching markets.

\section{Technical Lemmas}

\begin{lemma}{(Lemma 3 in \cite{wang2018thompson})}\label{lem:beta:concen}
	In Algorithm \ref{alg:CA-TS}, for any player $p_i$, arm $a_j$, and round $t$, we have
	\begin{align*}
		\PP{\abs{\theta_{i,j}(t)-\hat{\mu}_{i,j}(t)} > \varepsilon \mid a_{i,j}(t),b_{i,j}(t) } \le 2 \exp\bracket{-2N_{i,j}(t)\varepsilon^2}\,.
	\end{align*}
\end{lemma}

\begin{lemma}{(Chernorff-Hoeffding bound)}\label{lem:chernoff}
Let $X_1,X_2,\ldots,X_n$ be identical independent random variables such that $X_i \in [0,1]$ and $\EE{X_i} = \mu$ for any $i \in [n]$. Then for any $\varepsilon \ge 0$, we have
\begin{align*}
	\PP{\abs{\frac{1}{n} \sum_{i=1}^n X_i - \mu} \ge \varepsilon} \le  2\exp\bracket{-2n\varepsilon^2}\,.
\end{align*}
\end{lemma}

\begin{lemma}\label{fact:nips20}{(Lemma 5 in \cite{perrault2020statistical})}
In Algorithm \ref{alg:CA-TS}, for any player $p_i$ and arm $a_j$, let $\eta_q$ be the round $t$ at which $N_{i,j}(t+1)=q=N_{i,j}(t)+1$, we have
    \begin{align*}
        \EE{\sup_{t \ge \eta_q+1} \frac{1}{ \PP{\abs{ \theta_{i,j}(t)-\mu_{i,j} }\le \varepsilon \mid \cH_{t} }} }   -1 \le \begin{cases}
\bracket{ c\varepsilon^{-4} } & \text{for every } q \ge 0\\
e^{-\varepsilon^2 q /8}\bracket{ c'\varepsilon^{-4} } & \text{if } q>8/\varepsilon^2\,,
\end{cases}
    \end{align*} 
    where $c$ and $c'$ are two universal constants, $\cH_{t}$ is the history of observations at time $t$. 
\end{lemma}

\begin{lemma}\label{fact:gaussian:concen}{(Concentration and anti-concentration inequalities for Gaussian distributed random variables \cite{abramowitz1964handbook}.)}
For a Gaussian distributed random variable $Z$ with mean $m$ and variance $\sigma^2$, for any $z$,
    \begin{align*}
        \frac{1}{4\sqrt{\pi}}\exp\bracket{-\frac{7z^2}{2}}< \PP{\abs{Z-m} > z\sigma} \le \frac{1}{2}\exp\bracket{-\frac{z^2}{2}}\,.
    \end{align*}
\end{lemma}

\begin{lemma}\label{fact:subgaussian:concen}{(\cite[Corollary 5.5]{lattimore2020bandit})} Assume that $X_i - \mu$ are independent, $\sigma$-subgaussian random variables. Then for any $\varepsilon \ge 0$,
\begin{align*}
    \PP{\abs{\hat{\mu}-\mu}\ge \varepsilon} \le \exp\bracket{-\frac{n\varepsilon^2}{2\sigma^2}} \,,
\end{align*}
where $\hat{\mu} = \frac{1}{n}\sum_{t=1}^n X_t$.
\end{lemma}

\fi

\end{document}